%% file: main.tex
\documentclass{article}

\usepackage[utf8]{inputenc} % allow utf-8 input
\usepackage[T1]{fontenc}    % use 8-bit T1 fonts
\usepackage{url}            % simple URL typesetting
\usepackage{booktabs}       % professional-quality tables
\usepackage{amsfonts}       % blackboard math symbols
\usepackage{nicefrac}       % compact symbols for 1/2, etc. 
\usepackage{mathtools,stmaryrd,xparse}       % integer interval
\usepackage{microtype}      % microtypography
\usepackage{dsfont}
\usepackage{wrapfig}
\usepackage{amsmath}
\usepackage{mystyle}
\usepackage{comment}
\usepackage[numbers]{natbib}
\usepackage{amssymb}
\usepackage{graphicx}

\title{Distribution-Based Invariant Deep Networks \\ for Learning Meta-Features}

\author{
  Gwendoline De Bie\\
  TAU - LRI and ENS, \\ PSL University \\
  \texttt{debie@dma.ens.fr}
  \and
  Herilalaina Rakotoarison\\
  TAU - LRI, INRIA\\
  \texttt{heri@lri.fr}
  \and
  Gabriel Peyré\\
  CNRS and ENS, \\ PSL University \\
  \texttt{gabriel.peyre@ens.fr}
  \and
  Michele Sebag\\
  CNRS, Paris-Saclay University\\
  \texttt{sebag@lri.fr}
}

\graphicspath{ {images/} }

\begin{document}

\maketitle
\def\mf{meta-feature}
\def\Z{\mbox{\bf Z}}
\def\z{\mbox{\bf z}}
\def\u{\mbox{\bf u}}
\def\v{\mbox{\bf v}}
\def\x{\mbox{\bf x}}
\def\Z{\mbox{Z}}
\def\Teta{\mbox{\mathcal{Y}}}
\def\Surrogate{\mbox{\mathcal{S}}}
\def\G{\mbox{$G$}} % notation for S_{d_X} * S_{d_Y}
\def\H{\mbox{$H$}} % notation for a subgroup of S_d
\def\XX{{\sc Dida}}
\def\AS{{\sc Auto-Sklearn}}
\def\Baseline{{\sc Dataset2Vec}} % baseline
% measure notations
\def\MMd{\mbox{$\mathcal{P}(\RR^d)$}}
\def\MMIk{\mbox{$\mathcal{P}(I^k)$}}
\def\MMIdm{\mbox{$\mathcal{P}(I^{d_m})$}}
\def\MMr{\mbox{$\mathcal{P}(\RR^r)$}}
\def\MMO{\mbox{$\mathcal{P}(\Omega)$}}
\def\MMOq{\mbox{$\mathcal{P}(\Omega)_{/\sim}$}}

\input{sections/00-Abstract}

\input{sections/01-Intro}

\input{sections/02-Dida}

\input{sections/03-Theory}

\input{sections/04-Expe}

\input{sections/05-Conclu}

\input{sections/acknowledgements}

\newpage
\bibliographystyle{acm}
\bibliography{main}

\newpage
\appendix

\input{sections/App-01-continuous}

\input{sections/App-02-proofs}

\input{sections/App-03-expe}

\end{document}

%% file: sections/00-Abstract.tex
\begin{abstract}%
Recent advances in deep learning from probability distributions successfully achieve classification or regression from distribution samples, thus invariant under permutation of the samples.
The first contribution of the paper is to extend these neural architectures to achieve invariance under permutation of the features, too.  The proposed architecture, called \XX, inherits the NN properties of universal approximation, and its robustness w.r.t. Lipschitz-bounded transformations of the input distribution is established. 
The second contribution is to empirically and comparatively demonstrate the merits of the approach on two tasks defined at the dataset level. On both tasks, \XX\ learns meta-features supporting the characterization of a (labelled) dataset. The first task consists of predicting whether two dataset patches are extracted from the same initial dataset. The second task 
consists of predicting whether the learning performance achieved by a hyper-parameter configuration under a fixed algorithm (ranging in k-NN, SVM, logistic regression and linear SGD) dominates that of another configuration, for a dataset extracted from the OpenML benchmarking suite. On both tasks, \XX\ outperforms the state of the art: DSS \citep{maron2020learning} and \Baseline{} \citep{jomaa2019dataset2vec} architectures, as well as the models based on the hand-crafted meta-features of the literature. 
\end{abstract}

%% file: sections/01-Intro.tex
%%%%%%%%%%%%%%%%%%%%%%%%%%%%%%%%%%%%%%%%%%%%%%%%%%%
\section{Introduction}
\label{intro}

Deep networks architectures, initially devised for structured data such as images~\citep{Krizhevsky2012} and speech~\citep{hinton2012deep}, have been extended to enforce some invariance or equivariance properties~\citep{248459} for more complex data representations. Typically, the network output is required to be invariant with respect to permutations of the input points when dealing with point clouds~\citep{qi2017pointnet}, graphs \citep{HenaffBL15} or probability distributions \citep{de2019stochastic}. 
The merit of invariant or equivariant neural architectures is twofold. On the one hand, they inherit the universal approximation properties of neural nets~\citep{cybenko1989approximation,leshno1993multilayer}. On the other hand, 
the fact that these architectures comply with the requirements attached to the data representation yields more robust and more general models, through constraining the neural weights and/or reducing their number. 

\paragraph{Related works.}

Invariance or equivariance properties are relevant to a wide range of applications. In the sequence-to-sequence framework, one might want to relax the sequence order  \citep{VinyalsBK15}. When modelling dynamic cell processes, one might want to follow the cell evolution at a macroscopic level, in terms of distributions as opposed to, a set of individual cell trajectories \citep{hashimoto16}. In computer vision, one might want to handle a set of pixels, as opposed to a voxellized representation, for the sake of a  better scalability in terms of data dimensionality and computational resources \citep{de2019stochastic}. 

Neural architectures enforcing invariance or equivariance properties have been pioneered by \citep{qi2017pointnet,zaheer2017} for learning from point clouds subject to permutation invariance or equivariance. These have been extended to permutation equivariance across sets \citep{hartford2018deep}. Characterizations of invariance or equivariance under group actions have been proposed in the finite \citep{gens2014,cohenc16,ravanbakhsh17a} or infinite case \citep{wood1996representation,kondor2018generalization}. 

On the theoretical side, \citep{Maron2019a,keriven2019} have proposed a general 
characterization of linear layers enforcing invariance or equivariance properties with respect to the whole permutation group on the feature set. The universal approximation properties of such architectures have been established in the case of sets \citep{zaheer2017}, point clouds \citep{qi2017pointnet}, equivariant point clouds \citep{segol2019universal}, discrete measures \citep{de2019stochastic}, invariant \citep{Maron2019b} and equivariant \citep{keriven2019} graph neural networks. 
The approach most related to our work is that of \citep{maron2020learning}, handling point clouds and presenting a neural architecture invariant w.r.t. the ordering of points and their features. In this paper, the proposed {\em distribution-based invariant deep architecture} (\XX) extends \citep{maron2020learning} as it handles (discrete or continuous) probability distributions instead of point clouds. This enables to leverage the topology of the Wasserstein distance to provide more general approximation results, covering \citep{maron2020learning} as a special case.
%\footnote{Specifically, probability distributions come with optimal transport \citep{peyre2019computational} as  natural topology; in opposition, the natural topology on point clouds is the Haussdorf distance.} 

\paragraph{Motivations.}
A main motivation for \XX\ is the ability to characterize datasets through {\em learned meta-features}. Meta-features, aimed to represent a dataset as a vector of characteristcs, have been mentioned in the ML literature for over 40 years, in relation with several 
key ML challenges: (i) learning a performance model, predicting {\em a priori} the performance of an algorithm (and the hyper-parameters thereof) on a dataset \citep{Rice76,Wolpert:1996a, automl_book}; 
(ii) learning a generic model able of quick adaptation to new tasks, e.g. one-shot or few-shot, through the so-called meta-learning approach \citep{finn2018,yoon2018}; (iii) hyper-parameter transfer learning \citep{perrone2018}, aimed to transfer the performance model learned for a task, to another task. 
A large number of meta-features have been manually designed along the years \citep{smith-mile-MLJ18}, ranging from sufficient statistics to the so-called {\em landmarks} \citep{Pfahringer-ICML00}, computing the performance of fast ML algorithms on the considered dataset. Meta-features, expected to describe the joint distribution underlying the dataset, should also be inexpensive to compute. 
The learning of meta-features has been first tackled by \citep{jomaa2019dataset2vec} to our best knowledge, defining the \Baseline{} representation. Specifically, \Baseline{} is provided two 
patches of datasets, (two subsets of examples, described by two (different) sets of features), and is trained to predict whether those patches are extracted from the same initial dataset. 
%
%%%%
\paragraph{Contributions.}
The proposed \XX\ approach extends the state of the art \citep{maron2020learning,jomaa2019dataset2vec} in two ways. Firstly, it is designed to handle discrete or continuous probability distributions, as opposed to point sets (Section \ref{sec:dida}). As said, this extension enables to leverage the more general topology of the Wasserstein distance as opposed to that of the Haussdorf distance (Section~\ref{sec:theory}). This framework is used to derive theoretical guarantees of stability under bounded distribution transformations, as well as universal approximation results, extending \citep{maron2020learning} to the continuous setting.  
Secondly, the empirical validation of the approach on two tasks defined at the dataset level demonstrates the merit of the approach compared to the state of the art \citep{maron2020learning,jomaa2019dataset2vec, munoz2018instance} (Section~\ref{sec:expe}).

\paragraph{Notations.} $[m]$ denotes the set of integers $\{1, \ldots m\}$. Distributions, including discrete distributions (datasets) are noted in bold font. Vectors are noted in italic, with $x[k]$ denoting the $k$-th coordinate of vector $x$. %$\rho$ is an activation function.

%% file: sections/02-Dida.tex
\section{Distribution-Based Invariant Networks for Meta-Feature Learning}
\label{sec:dida}

This section describes the core of the proposed distribution-based invariant neural architectures, 
specifically the mechanism of mapping a point distribution onto another one subject to sample and feature invariance, referred to as {\em invariant layer}. 
%It details how such invariant layers can be trained to perform regression or classification with customized invariance requirements. 
For the sake of readability, this section focuses on the case of discrete distributions, referring the reader to Appendix \ref{appendix:notations} for the general case of continuous distributions. 

\subsection{Invariant Functions of Discrete Distributions}
\label{notations}
Let  \z $ = \{ (x_i,y_i) \in \RR^d, i \in [n] \}$ denote a dataset including $n$ labelled samples, with $x_i \in \RR^{d_X}$ an instance and $y_i \in \RR^{d_Y}$ the associated multi-label. With $d_X$ and $d_Y$ respectively the dimensions of the instance and label spaces, let $d \eqdef d_X+d_Y$. By construction, $\z$ is invariant under permutation on the sample ordering; it is viewed as an $n$-size discrete distribution $\frac{1}{n}\sum_{i=1}^n \de_{z_i}$ in $\RR^d$ with $\de_{z_i}$ the Dirac function at $z_i$.
%, as opposed to a point cloud.   
% XXXXXXXXXXXXXX Ici, ca serait super de mettre un exemple clarifiant la différence. En footnote
In the following, $\Z_n(\RR^d)$ denotes the space of such  $n$-size point distributions, with $\Z(\RR^d) \eqdef \cup_n \Z_n(\RR^d)$ the space of distributions of arbitrary size. 
% Contrary to $n-$sized uniform samples in $\Z_n(\RR^d)$, in the general case, a (discrete or uniform) distribution (with support in $\RR^d$) belongs to $\Z(\RR^d).$
% besoin de cela pour l'équation 3 + pour distinguer le cas discret/continu

%Plus besoin de ca; plus d'AutoML.
%As the performance of an ML algorithm is most generally invariant w.r.t. permutations operating on the feature or label spaces, the neural architectures leveraged to learn the meta-features must enjoy the same property. 
Let $\G\eqdef S_{d_X} \times S_{d_Y}$ denote the group of permutations independently operating on the feature and label spaces. For $\sigma=(\sigma_X,\sigma_Y) \in \G$, 
the image $\si(z)$ of a labelled sample is defined as $(\si_X(x),\si_Y(y))$, with $x=(x[k], k \in [d_X])$ and $\si_X(x) \eqdef (x[\si_X^{-1}(k)], k \in [d_X])$. 
For simplicity and by abuse of notations, the operator mapping a distribution $\z=(z_i, i \in [n])$ to $\{\si(z_i)\}\eqdef \sigma_\sharp \z$ is still denoted $\si$.

Let $\Z(\Om)$ denote the space of distributions supported on some domain $\Om \subset \RR^d$, with $\Omega$ invariant under permutations in $\G$.
The goal of the paper is to define and train deep architectures, implementing functions $\phi$ on $\Z(\Omega \subset \RR^d)$ that are invariant under $\G$, i.e. such that $\forall \sigma \in \G, \phi(\si_\sharp\z)=\phi(\z)$\footnote{As opposed to $\G$-\emph{equivariant} functions that are characterized by $\forall \sigma \in \G, \phi(\si_\sharp\z)=\si_\sharp\phi(\z)$}.
By construction, a multi-label dataset is invariant under permutations of the samples, of the features, and of the multi-labels. Therefore, any meta-feature, that is, a feature describing a multi-label dataset, is required to satisfy the above property. 

\subsection{Distribution-Based Invariant Layers}\label{inv-sto-lay}

The building block of the proposed architecture, the invariant layer meant to satisfy the feature and label invariance requirements, is defined as follows, taking 
inspiration from \citep{de2019stochastic}.

\begin{defn}(Distribution-based invariant layers)
Let an interaction functional $\phi: \RR^d\times\RR^d \to \RR^r$ be \G-invariant:
$$
    \forall \sigma \in \G, \quad
    \forall (z_1,z_2) \in \RR^d \times \RR^d, \quad \phi(z_1,z_2)=\phi(\sigma(z_1),\sigma(z_2)).
$$ 
The distribution-based invariant layer $f_{\phi}$ is defined as
\begin{align} \label{def} 
f_{\phi} : \z=(z_i)_{i \in [n]} \in \Z(\RR^d) \mapsto f_\phi(\z) \eqdef \left[ \frac{1}{n}\sum_{j=1}^n \phi(z_1,z_j), \hdots, \frac{1}{n}\sum_{j=1}^n \phi(z_n,z_j) \right] \in\Z(\RR^r).
\end{align}
\end{defn}

It is easy to see that $f_\phi$ is \G-invariant. The construction of $f_\phi$ is extended to the general case of possibly continuous probability distributions by essentially replacing sums by integrals (Appendix \ref{appendix:notations}).

\begin{rem} \label{rem:varying_n}(Varying sample size $n$). By construction, $f_\phi$ is defined on $\Z(\RR^d)=\cup_n\Z_n(\RR^d)$ (independent of $n$), such that it supports inputs of arbitrary cardinality $n$.
\end{rem}

\begin{rem}(Discussion w.r.t. \citep{maron2020learning}) 
%Contrary to \citep{maron2020learning}, the 
The above definition of $f_\phi$ is based on the aggregation of pairwise terms $\phi(z_i,z_j)$. The motivation for using a pairwise $\phi$ is twofold. On the one hand, capturing local sample interactions allows to create more expressive architectures, which is important to improve the
performance on some complex data sets, as %exemplified 
illustrated in the experiments (Section \ref{sec:expe}). On the other hand, interaction functionals are crucial to design universal
architectures (Appendix \ref{appendix:thm}, theorem \ref{thm:sw}). 
The proposed theoretical framework relies on the Wasserstein distance (corresponding to the convergence in law of probability distributions), which enables to compare distributions with varying number of points or even with continuous densities. In %sharp 
contrast, \cite{maron2020learning} do not use interaction functionals, and establish the universality of their DSS architecture for fixed dimension $d$ and number of points $n$. Moreover, DSS happens to resort to 
max pooling operators, discontinuous w.r.t. the Wasserstein topology (see Remark \ref{rem:compmaron}). % (Gab)
%
%Moreover, these results involve a more general distributional setting and topology, relying on the Wasserstein metric (section \ref{subsec:ot}), that is incompatible with the use of e.g., max pooling aggregators in \citep{maron2020learning}, due to their being not continuous.
%
%The comparative generality of the results is further discussed in Remark \ref{rem:compmaron}.
\end{rem}

Two particular cases are when $\phi$ only depends on its first or second input: 
\begin{itemize}
    \item[(i)] if $\phi(z,z')=\psi(z')$, then $f_\phi$ computes a global ``moment'' descriptor of the input, as $f_\phi(\z)=\frac{1}{n} \sum_{j=1}^n \psi(z_j) \in \RR^r$.
    \item[(ii)] if $\phi(z,z')=\xi(z)$, then $f_\phi$ transports the input distribution via $\xi$, as 
% Erreur: RR^r --> Z_n(\RR^r)
%    $f_\phi(\z)=\{ \xi(z_i)  \}_i \subset \RR^r$. This operation is referred to as a {\em push-forward}.
    $f_\phi(\z)=\{ \xi(z_i), i \in [n]  \} \in Z(\RR^r)$. This operation is referred to as a {\em push-forward}.
\end{itemize}
%  In the architectures, this flexibility allows to deal with distributions among several layers (see \citep{de2019stochastic}) while performing invariant regression (see section \ref{learning-mf}).

\begin{rem}\label{rem:varyingspaces}(Varying dimensions $d_X$ and $d_Y$). 
Both in practice and in theory, it is important that $f_\phi$ layers (in particular the first layer of the neural architecture) handle datasets of arbitrary number of features $d_X$ and number of multi-labels $d_Y$. The proposed approach, used in the experiments (Section \ref{sec:expe}), is to define $\phi$ on the top of a four-dimensional aggregator, as follows. Letting  $z = (x,y)$ and $z'=(x',y')$ be two samples in $\RR^{d_X} \times \RR^{d_Y}$, let $u$ be defined from $\RR^4$ onto $\RR^t$, consider the sum of $u(x[k],x'[k],y[\ell],y'[\ell])$ for $k$ ranging in $[d_X]$ and $\ell$ in $[d_Y]$, and apply mapping $v$ from $\RR^t$ to $\RR^r$ on the sum:
\begin{equation*}
 \phi(z,z')=v\left( \sum_{k=1}^{d_X}\sum_{\ell=1}^{d_Y} u(x[k],x'[k],y[\ell],y'[\ell]) \right)   
\end{equation*}
%
% $\phi$ is usually defined to be applicable to points defined on space of arbitrary dimension $d \in [1;d_m]$. 
%
% In practice, this is achieved by summing over (an arbitrarily large number of) feature coordinates: $\phi(\z_i,\z_j)=\rho\left( \sum_k u(\z_i^{(k)},\z_j^{(k)}) \right)$, where $u:\RR^2\to\RR^{t}$ and $\rho:\RR^t\to\RR^r$ are independent of $d$. 
% With this choice of $\phi$, a mapping of type (\ref{def}) can be the first layer of a network processing a large variety of datasets.
\end{rem}

\begin{rem}\label{rem:tenso}(Localized computation)
In practice, the quadratic complexity of $f_\phi$ w.r.t. the number $n$ of samples can be reduced by only computing $\phi(z_i,z_j)$ for pairs $z_i,z_j$ sufficiently close to each other.  Layer $f_\phi$ thus extracts and aggregates information related to the neighborhood of the samples. 
\end{rem}

\subsection{Learning Meta-features} \label{learning-mf}

%\paragraph{Meta-features learning via invariant regression.} 

%\todo{Is it important to have different non-linearity? Can't we use the same $\rho$ everywhere for simplicity?}

The proposed distributional neural architectures defined on point distributions (\XX{}) are sought as
\begin{equation} \label{eq-defn-dida}
\z \in \Z(\RR^{d}) \mapsto \Ff_\zeta(\z)\eqdef f_{\phi_m} \circ f_{\phi_{m-1}} \circ \hdots \circ f_{\phi_1}(\z) \in \RR^{d_{m+1}}
\end{equation} 
where $\zeta$ are the trainable parameters of the architecture (below).
Only the case $d_Y=1$ is considered in the remainder.
The $k$-th layer is built on the top of $\phi_k$, mapping pairs of vectors in $\RR^{d_k}$ onto $\RR^{d_{k+1}}$, with $d_1=d$ (the dimension of the input samples). Last layer is built on $\phi_m$, only depending on its second argument; it maps the distribution in layer $m-1$ onto a vector, whose coordinates are referred to as meta-features.

The \G-invariance and dimension-agnosticity of the whole architecture only depend on the first layer $f_{\phi_1}$ satisfying these properties.  In the first layer, $\phi_1$ is sought as $\phi_1((x,y),(x',y')) = v(\sum_k u(x[k],x'[k],y,y'))$ (Remark~\ref{rem:varyingspaces}), with  $u(x[k],x'[k],y,y') = (\rho(A_u \cdot (x[k];x'[k]) + b_u, \mathds{1}_{y\neq y'})$ in $\RR^{t}\times\{0,1\}$, where $\rho$ is a non-linear activation function,  
%% RELU ou sigmoid ?
$A_u$ a $(2,t)$ matrix, $(x[k];x'[k])$ the 2-dimensional vector concatenating $x[k]$ and $x'[k]$, and $b_u$ a $t$-dimensional vector. 
With $e = \sum_k u(x[k],x'[k],y,y'))$, function $v$ likewise applies a non-linear activation function $\rho$ on an affine transformation of $e$: $v(e) = \rho(A_v \cdot e+b_v)$, with $A_v$ a $(t,r)$ matrix and $b_v$ a $r$-dimensional vector.

Note that the subsequent layers need neither be invariant w.r.t. the number of samples, nor handle a varying number of dimensions. Every $\phi_k, k \ge 2$ is defined as $\phi_k = \rho(A_k \cdot +b_k)$, with $\rho$ an activation function, $A_k$ a $(d_k,d_{k+1})$ matrix and $b_k$ a $d_{k+1}$-dimensional vector. 
The \XX\ neural net thus is parameterized by $\zeta \eqdef (A_u,b_u,A_v,b_v,\{A_k,b_k\}_k)$, that is classically learned by stochastic gradient descent from the loss function defined after the task at hand (Section~\ref{sec:expe}). %The non-linear activation function $\rho$ is set to RELU in the experiments.
%
%By construction, these architectures are invariant w.r.t. the orderings of both the points composing the input distributions and their coordinates. Note that the input distributions can be composed of any number of points in any dimension.

% Since both in theory and in practice, distribution-based invariant layers allow for distributions with varying supports (see remarks \ref{rem:varyingspaces} and \ref{rem:spaces}),  \XX{} supports the comparison of point distributions in different spaces, through the meta-features.

%% file: sections/03-Theory.tex
\section{Theoretical Analysis}
\label{sec:theory}
This section analyzes the properties of invariant-layer based neural architectures, specifically their robustness w.r.t. bounded transformations of the involved distributions, and their approximation abilities w.r.t. the convergence in law, which is the natural topology for distributions. As already said, the discrete distribution case is considered in this section for the sake of readability, referring the reader to Appendix \ref{appendix:notations} for the general case of continuous distributions. 
\subsection{Optimal Transport Comparison of Datasets}\label{subsec:ot}

\paragraph{Point clouds vs. distributions.}
Our claim is that datasets should be seen as probability distributions, rather than point clouds. Typically, including many copies of a point in a dataset amounts to increasing its importance, which usually makes a difference in a standard machine learning setting. 
Accordingly, the topological framework used to define and learn meta-features in the following is that of the convergence in law, with the distance among two datasets being quantified using the Wasserstein distance (below). In contrast, the point clouds setting (see for instance \citep{qi2017pointnet}) relies on the Haussdorff distance among sets to theoretically assess the robustness of these architectures. While it is standard for 2D and 3D data involved in graphics and vision domains, it faces some limitations in higher dimensional domains, e.g. due to max-pooling being a non-continuous operator w.r.t. the convergence in law topology. 

\begin{comment}
\begin{figure}[]
\centering
\begin{minipage}{.5\textwidth}
  \centering
  \includegraphics[scale=0.4]{img/d_H.png}
  \\ $d_H(\D^{(1)},\D^{(2)})=1.25$,
  \\ $\Wass_1(\D^{(1)},\D^{(2)})=1.20$
\end{minipage}\hspace*{0.01cm} 
\begin{minipage}{.5\textwidth}
  \centering
  \includegraphics[scale=0.4]{img/d_H_dot.png}
  \\ $d_H(\D^{(1)}+\{\text{dot}\},\D^{(2)})$ unchanged, 
  \\ $\Wass_1(\D^{(1)}+\{\text{dot}\},\D^{(2)})=1.22$
\end{minipage}
\caption{\todo{Gab: I am not sure this figure is very much helpful. }A difference between $d_H$ and $\Wass_1$: adding the square dot (right) changes the $1-$Wasserstein distance while the Hausdorff distance remains equal between the two settings.}
\label{fig:HvsW}
\end{figure}
\end{comment}

\paragraph{Wasserstein distance.} Referring the reader to \citep{santambrogio2015optimal,peyre2019computational} for a more comprehensive presentation, the standard $1$-Wasserstein distance between two discrete probability distributions $\z,\z' \in \Z_n(\RR^d) \times \Z_m(\RR^d)$ is defined as:
% $\z^{(1)}=(\z_1^{(1)},\hdots,\z_n^{(1)}) \in \Z_n(\RR^d)$ and $\z^{(2)}=(\z_1^{(2)},\hdots,\z_m^{(2)}) \in \Z_m(\RR^d)$:
\begin{equation*}
	\Wass_1(\z,\z') \eqdef 
	\umax{f \in \Lip_1(\RR^d)}
	    \frac{1}{n}\sum_{i=1}^n f(z_i) -
	    \frac{1}{m}\sum_{j=1}^m f(z_j')
	% \umin{\pi_{ij}\geq 0} \sum_{i,j} \pi_{ij} \norm{\z_i^{(1)}-\z_j^{(2)}}
\end{equation*}
with $\Lip_1(\RR^d)$ the space of $1$-Lipschitz functions $f:\RR^d \rightarrow \RR$.
% subject to $\sum_j \pi_{ij}=\frac{1}{n}$ and $\sum_i \pi_{ij}=\frac{1}{m}$, such that $\z^{(1)}$ and $\z^{(2)}$ are the two marginal distributions of the coupling $\pi$. 
To account for the invariance requirement (making indistinguishable $\z=(z_1,\hdots,z_n)$ and its permuted image $(\sigma(z_1),\hdots,\sigma(z_n)) \eqdef \sigma_\sharp \z$ under $\sigma \in \G$), we introduce the $\G$-invariant $1$-Wasserstein distance: for $\z \in \Z_n(\RR^d), \z' \in \Z_m(\RR^d)$:
\begin{equation*}
\overline{\Wass}_1(\z,\z')=\min_{\sigma \in \G} \Wass_1(\sigma_\sharp \z,\z')
\end{equation*}
such that $\overline{\Wass}_1(\z,\z')=0$ if and only if $\z$ and $\z'$ belong to the same equivalence class (Appendix \ref{appendix:notations}), i.e. are equal in the sense of probability distributions up to sample and feature permutations.

\paragraph{Lipschitz property.} In this context, a map $f$ from $\Z(\RR^d)$ onto $\Z(\RR^r)$ %, taking an any size discrete probability distribution on $\RR^d$ as input. 
%Ca devient ridicule de le mettre partout. --> d'accord!
%(see  Appendix A for continuous distributions). 
is continuous for the convergence in law (a.k.a. weak convergence on distributions, denoted $\rightharpoonup$) iff for any sequence $\z^{(k)} \rightharpoonup \z$, then $f(\z^{(k)}) \rightharpoonup f(\z)$. 
The Wasserstein distance metrizes the convergence in law, in the sense that $\z^{(k)} \rightharpoonup \z$ is equivalent to $\Wass_1(\z^{(k)},\z) \rightarrow 0$.
Furthermore,  map $f$ is said to be $C$-Lipschitz for the permutation invariant $1$-Wasserstein distance iff 
\begin{align} \label{lipdef}
%    \forall (\z,\z') \in \Z(\RR^d)\times \Z(\RR^d), \quad
    \forall \z,\z' \in \Z(\RR^d), \quad
	\overline{\Wass}_1( f(\z),f(\z') )
	\leq 
	C \overline{\Wass}_1(\z,\z').
\end{align}
The $C$-Lipschitz property entails the continuity of $f$ w.r.t. its input: if two input distributions are close in the permutation invariant $1$-Wasserstein sense, the corresponding outputs are close too.

\subsection{Regularity of Distribution-Based Invariant Layers} \label{subsec:reg} 

Assuming the interaction functional to satisfy the Lipschitz property:
\begin{equation}\label{eq-condition-Reg}
% \tag{$\text{Reg}_\phi$}
\forall z \in \RR^d, \quad
\phi(z,\cdot)
\quad\text{and}\quad
\phi(\cdot,z)
\quad\text{are}\quad C_\phi-\text{Lipschitz.}
\end{equation}
the robustness of invariant layers with respect to different variations of their input is established (proofs in Appendix~\ref{appendix:reg}). We first show that invariant layers also satisfy Lipschitz property, ensuring that deep architectures of the form (\ref{eq-defn-dida}) map close inputs onto close outputs. 

\begin{prop} \label{prop:lip}
%If interaction functional $\phi$ is Lipschitz:  
%\begin{equation}\label{eq-condition-Reg}
% \tag{$\text{Reg}_\phi$}
%\forall z \in \RR^d, \quad
%\phi(z,\cdot)
%\quad\text{and}\quad
%\phi(\cdot,z)
%\quad\text{are}\quad \Lip(\phi)-\text{Lipschitz.}
%\quad\text{are}\quad C_\phi-\text{Lipschitz.}
%\end{equation}
%Invariant layers $f_\phi$ of type (\ref{def}) are $(2r\Lip(\phi))$-Lipschitz in the sense of 
Invariant layer $f_\phi$ of type (\ref{def}) is $(2r C_\phi)$-Lipschitz in the sense of (\ref{lipdef}).
\end{prop}
%\begin{proof} In Appendix~\ref{appendix:reg}. This result ensures the continuity of deep architectures of the form (\ref{eq-defn-dida}). 
%\end{proof}
% This ensures the desired trait that close distributions have close outputs by $f_\phi$. Hence all architectures composed of such layers are as well. 

%\todo{I did not understand this statement with auto-encoder. Later $\tau : \RR^d \mapsto \RR^d$ is supposed to be close to identity, so that $\tau_\sharp \z$ cannot be the latent representation}

A second result regards the case where two datasets $\z$ and $\z'$ are such that $\z'$ is the image of $\z$ through some diffeomorphism $\tau$ ($\z=(z_1,\hdots,z_n)$ and  $\z' = \tau_\sharp\z=(\tau(z_1),\hdots,\tau(z_n))$.  
If $\tau$ is close to identity, then the following proposition shows that $f_\phi(\tau_\sharp\z)$ and $f_\phi(\z)$ are close too. More generally, if continuous transformations $\tau$ and $\xi$ respectively apply on the input and output space of $f_\phi$, and are close to identity, then $\xi_\sharp f_\phi(\tau_\sharp \z)$ and $f_\phi(\z)$ are also close.

% :\RR^d\to\RR^d$ (resp. $\tau:\RR^r\to\RR^r$), then the output are also close.

\begin{prop} \label{prop:lipAE}
Let $\tau : \RR^d \rightarrow \RR^d$ and $\xi : \RR^r \rightarrow \RR^r$ be two Lipschitz maps with respectively Lipschitz constants $C_\tau$ and $C_\xi$. Then,  %\todo{Gab: is it possible to add both $\xi$ and $\tau$ in the first equation?}
\begin{align*}
  \forall &\z \in \Z(\Omega),~  \overline{\Wass}_1(\xi_\sharp f_\phi(\tau_\sharp\z),f_\phi(\z)) \leq \sup_{x \in f_\phi(\tau(\Omega))} \norm{\xi(x)-x}_2 + 2r\Lip(\phi)\sup_{x\in\Omega} \norm{\tau(x)-x}_2 \\
    %%%%
   &\forall \z,\z' \in \Z(\Omega),   \text{ if $\tau$ is equivariant,} \hspace{0.2cm}
    \overline{\Wass}_1(\xi_\sharp f_\phi(\tau_\sharp\z),\xi_\sharp f_\phi(\tau_\sharp\z')) 
    \leq 
    2r\, C_\phi \,C_\tau \, C_\xi \overline{\Wass}_1(\z,\z') \quad 
\end{align*}
%
% If $\tau : \RR^d \rightarrow \RR^d$ is Lipschitz, then for all $\z,\z' \in \Z(\Omega)^2$, both $\overline{\Wass}_1(f_\phi(\tau_\sharp\z),f_\phi(\tau_\sharp\z')) \leq 2r \Lip(\phi) \Lip(\tau) \overline{\Wass}_1(\z,\z')$ and $\overline{\Wass}_1(\tau_\sharp f_\phi(\z),\tau_\sharp f_\phi(\z')) \leq 2r \Lip(\phi) \Lip(\tau) \overline{\Wass}_1(\z,\z')$ hold.
\end{prop}

\subsection{Universality of Invariant Layers}\label{sec:approx-thms}
Lastly, the universality of the proposed architecture is established, showing that the composition of an invariant layer (\ref{def}) and a fully-connected layer is enough to enjoy the universal approximation property, over all functions defined on $Z(\RR^d)$ with dimension $d$ less than some $D$ %(in the sense of
(Remark \ref{rem:varyingspaces}).

% \paragraph{Universality of Invariant Layers.}

% This theorem shows that any continuous invariant regression map on distributions can be approximated to arbitrary precision by the invariant composition of continuous functions. 

% Note that this result is fairly general: it holds for all compactly supported distributions (both continuous and discrete, regardless of their cardinality, as there is not dependence in $n$) and can be extended to (uniformly) suit measures supported on different spaces (see remark \ref{rem:spaces}).It also extends to product groups (see remark \ref{rem:product}), encompassing $G=S_d\times S_1$ as a special case, which is our experiments setting (see Section \ref{sec:expe}).

\begin{thm}\label{thm:approx-main}
    Let $\Ff: \Z(\Omega) \to \RR$ be a \G-invariant map on a compact $\Omega$, continuous for the convergence in law.
    Then $\forall \epsilon > 0$, there exists two continuous maps $\psi, \phi$ such that
    \begin{equation*}
        \forall \z \in \Z(\Omega), \quad
        \lvert \Ff(\z)- \psi \circ f_\phi(\z) \rvert < \epsilon 
    \end{equation*}
    where $\phi$ is \G-invariant and independent of $\Ff$.
\end{thm}

\begin{proof} The sketch of the proof is as follows (complete proof in Appendix~\ref{appendix:thm}). Let us define $\phi=g\circ h$ where: (i) $h$ is the collection of $d_X$ elementary symmetric polynomials in the features and $d_Y$ elementary symmetric polynomials in the labels, which is invariant under $G$; (ii) a discretization of $h(\Omega)$ on a grid is then considered, achieved thanks to $g$ that aims at collecting integrals over each cell of the discretization; (iii) $\psi$ applies function $\Ff$ on this discretized measure; this requires $h$ to be bijective, and is achieved by $\Tilde{h}$, through a projection on the quotient space $S_d/G$ and a restriction to its image compact $\Omega'$.
To sum up, $f_\phi$ defined as such computes an expectation which collects integrals over each cell of the grid to approximate measure $h_\sharp\z$ by a discrete counterpart $\widehat{h_\sharp\z}$. 
Hence $\psi$ applies $\Ff$ to $\Tilde{h}^{-1}_\sharp(\widehat{h_\sharp\z})$. Continuity is obtained as follows: (i) proximity of $h_\sharp\z$ and $\widehat{h_\sharp\z}$ follows from Lemma \ref{lem-old} in \citep{de2019stochastic}) and gets tighter as the grid discretization step tends to 0; (ii) Map $\Tilde{h}^{-1}$ is $1/d$-Hölder, after Theorem 1.3.1 from \citep{Rahman2002AnalyticTO}); therefore Lemma \ref{lem-holder} entails that $\overline{\Wass}_1(\z,{\Tilde{h}^{-1}}_\sharp\widehat{h_\sharp\z})$ can be upper-bounded; (iii) since $\Omega$ is compact, by Banach-Alaoglu theorem, $\Z(\Omega)$ also is. Since $\Ff$ is continuous, it is thus uniformly weakly continuous: choosing a discretization step small enough ensures the result. 
\end{proof}

\begin{rem}(Comparison with \citep{maron2020learning}) \label{rem:compmaron}
The above proof holds for functionals of arbitrary input sample size $n$, as well as continuous distributions, 
%making results in \citep{maron2020learning} a particular case of ours. 
% non, trop agressif !
generalizing results in \citep{maron2020learning}.
Note that 
the two types of architectures radically differ (more in Section \ref{sec:expe}).
\end{rem}

\begin{rem}\label{rem:approx}(Approximation by an invariant NN)
After theorem \ref{thm:approx-main}, any invariant continuous function defined on distributions with compact support can be approximated with arbitrary precision by an invariant neural network (Appendix \ref{appendix:thm}). The proof involves mainly three steps: (i) an invariant layer $f_\phi$ can be approximated by an invariant network; (ii) the universal approximation theorem \citep{cybenko1989approximation,leshno1993multilayer}; (iii) uniform continuity is used to obtain uniform bounds.
\end{rem}

\begin{rem}\label{rem:spaces}(Extension to different spaces)
%\todo{Gab: this remark is unclear, in what sense is the architecture of the theorem compatible with Remark \ref{rem:varyingspaces} and can thus be applied to space of varying dimension ?}
Theorem \ref{thm:approx-main} also extends to distributions supported on different spaces, via embedding them into a high-dimensional space. Therefore, any invariant function on distributions with compact support in $\RR^d$ with $d \le D$ can be uniformly approximated by an invariant network (Appendix \ref{appendix:thm}).
%ne mange pas le suspense: si le reviewer veut savoir la preuve, il va voir. Si tu donnes des indications de style padder avec des 0, le seul effet est de lui faire dire, ah ben c'est pas si malin. --> ok je suis d'accord!
\end{rem}

%% file: sections/04-Expe.tex
\def\AppendRes{Appendix \ref{appendix:expe}}
%%%%%%%%%%%%%%%%%%%%%%%%%%%%%%%%%%%%%%%%%%%%%%%%%%%%
\section{Experimental validation} \label{sec:expe}
%\todo{Ajouter fig archi + fig exp + resultat 38 + visu des metafeatures}

\begin{figure}[ht!]
    \centering
    \includegraphics[width=1\textwidth]{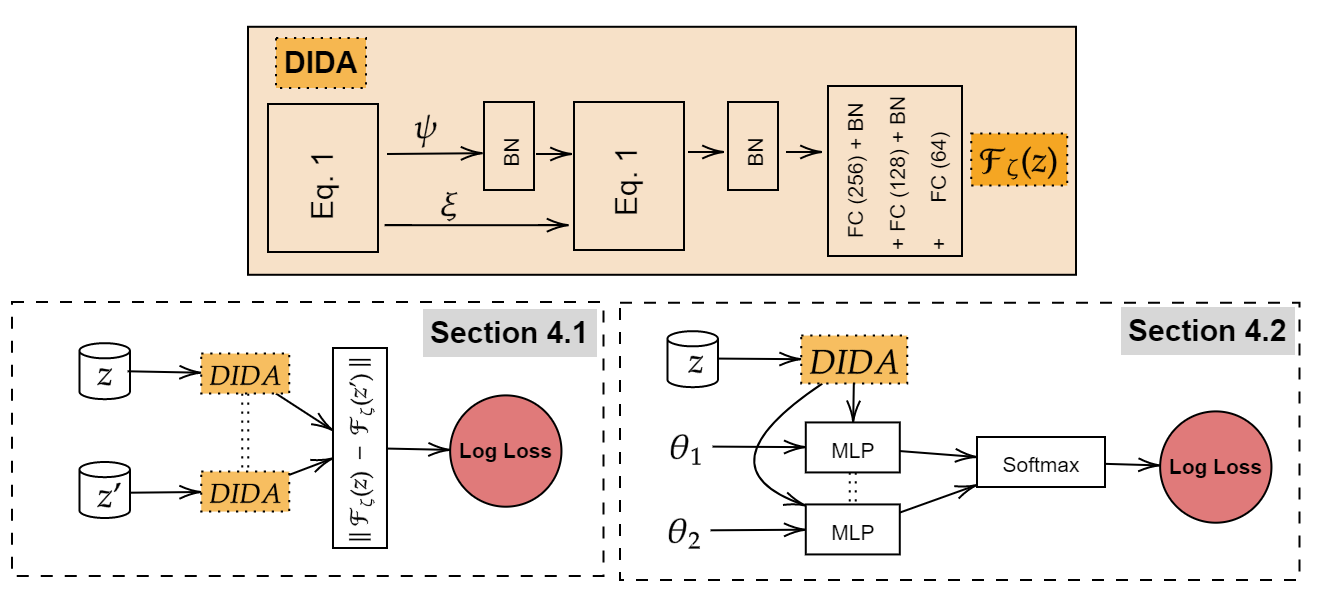}
    \caption{Learning meta-features with \XX. Top: the \XX{} architecture (BN stands for batch norm; FC for fully connected layer). Bottom left: Learning meta-features for patch identification using a Siamese architecture (section \ref{subsec:patchid}). Bottom right: learning meta-features for performance modelling, specifically to rank two hyper-parameter configurations $\theta_1$ and $\theta_2$ (section \ref{subsec:perflearning}).}
    \label{fig:expes_archi}
\end{figure}

The experimental validation presented in this section considers two goals of experiments: (i) assessing the ability of \XX{} to learn accurate meta-features; (ii)  assessing the merit of the \XX\ invariant layer design, building invariant $f_\phi$ on the top of an interactional function $\phi$ (Eq. \ref{def}). As said, this architecture is expected to grasp contrasts among samples, e.g. belonging to different classes; the proposed experimental setting aims to empirically investigate this conjecture.
These goals of experiments are tackled by comparing \XX\ to three baselines: DSS layers \citep{maron2020learning}; hand-crafted meta-features (HC) \citep{smith-mile-MLJ18} (Table \ref{tab:mf_handcrafted} in \AppendRes{}); \Baseline{} \citep{jomaa2019dataset2vec}. We implemented DSS (the code being not available) using linear and non-linear invariant layers.\footnote{The code source of \XX\ and (our implementation of) DSS is available in \AppendRes{}.}. All compared systems are allocated ca the same number of parameters. \\

%

%% Je pense qu'il faut enlever le \zeta ici. -> sans zeta, section pas claire pour des gens pas automl

\noindent{\bf Experimental setting.} Two tasks defined at the dataset level are considered: patch identification (section \ref{subsec:patchid}) and performance modelling (section \ref{subsec:perflearning}). On both tasks, the same \XX\ architecture is considered (Fig \ref{fig:expes_archi}), involving 2 invariant layers followed by 3 fully connected (FC) layers. Meta-features $\Ff_\zeta(\z)$ consist of the output of the third FC layer, with $\zeta$ denoting the trained \XX\ parameters. All experiments run on 1 NVIDIA-Tesla-V100-SXM2 GPU with 32GB memory, using Adam optimizer with base learning rate $10^{-3}$. %and batch size 32. 

\subsection{Task 1: Patch Identification} \label{subsec:patchid}
\def\u{\mbox{\bf u}}
The patch identification task consists of detecting whether two blocks of data are extracted from the same original dataset \citep{jomaa2019dataset2vec}. Letting $\u$ denote a $n$-sample, $d$-dimensional dataset, a patch $\z$ is constructed from $\u$ by retaining samples with index in $I \subset [n]$ and features with index in $J \subset [d]$. To each pair of patches $\z, \z'$ with same number of instances, is associated a binary meta-label $\ell(\z,\z')$ set to 1 iff \z\ and \z'\ are extracted from the same initial dataset $\u$.
\XX\ parameters $\zeta$ are trained to minimize the cross-entropy loss of model $\hat \ell_\zeta(\z,\z')=\exp\left( -\norm{\Ff_\zeta(\z)-\Ff_\zeta(\z')}_2 \right)$, with $\Ff_\zeta(\z)$ and $\Ff_\zeta(\z')$ the meta-features computed for $\z$ and $\z'$:
%(output of the last FC layer, Fig \ref{fig:expes_archi} top right):
\begin{equation}\label{eq:binaryloss}
    \mbox{Minimize ~} {\cal L}(\zeta) = -
    \sum_{\z,\z'} %w 
    \ell(\z,\z') \log(\hat \ell_\zeta(\z,\z')) + (1-\ell(\z,\z')) \log(1-\hat \ell_\zeta(\z,\z'))
\end{equation}
The classification results on toy datasets and UCI datasets (Table \ref{tab:res_batch_ident}, detailed in \AppendRes{}) show the pertinence of the \XX{} meta-features, particularly on the UCI datasets where the number of features widely varies from one dataset to another.  The relevance of the interactional invariant layer design is established on this problem as \XX\ outperforms both \Baseline{} and \textsc{DSS}.   
%
%Pairwise functionals are successful in increasing expressiveness of the learnt meta-features (Remark \ref{rem:compmaron}).
%
%Percentages are computed using 0,5-thresholding.
%
%Uncertainty estimates are obtained with 3 folds splitting of the test set.

\begin{table}[!ht]
   \centering
   \begin{tabular}{|c|c|c|c|c|}\toprule
     Method      & TOY                  & UCI\\ \midrule
     \Baseline{}$(^*)$    & 96.19 $\% \pm$ 0.28                           &  77.58 $\% \pm$ 3.13\\ \hline
     \textsc{DSS} layers (Linear aggregation)            & 89.32 $\% \pm$ 1.85                                   & 76.23 $\% \pm$ 1.84 \\
    \textsc{DSS} layers (Non-linear aggregation)            & 96.24 $\% \pm$ 2.04                                   & 83.97 $\% \pm$ 2.89\\
    \textsc{DSS} layers (Equivariant+invariant)            & 96.26 $\% \pm$ 1.40                                   & 82.94 $\% \pm$ 3.36 \\  
    \hline
    \XX{}          & \textbf{97.2} \% $\pm$ \textbf{0.1}           &  \textbf{89.70} \% $\pm$ \textbf{1.89}\\ 
    \bottomrule 
   \end{tabular}
   \caption{Patch identification: performance on 10 runs of \XX{}, \textsc{DSS} layers and  \Baseline{}. ($^*$): values reported from \citep{jomaa2019dataset2vec}.}
   \label{tab:res_batch_ident}
\end{table}

\subsection{Task 2: Performance model learning} \label{subsec:perflearning}
\def\MF{\mbox{$\Ff_\zeta(\z)$}}
The performance modelling task aims to assess {\em a priori} the accuracy of the classifier learned from a given machine learning algorithm with a given configuration $\theta$ (vector of hyper-parameters ranging in a hyper-parameter space $\Theta$, \AppendRes{}), on a dataset \z\ (for brevity, the performance of $\theta$ on $\z$) \citep{Rice76}. 
%The challenge of extracting a single set of meta-features relevant to modelling the performance of different ML algorithms is tackled along several tasks; each task considers a distinct ML algorithm, ranging in Logistic regression (LR), SVM, K-Nearest Neighbours (KNN), linear classifier learned with stochastic gradient descent (SGD). 
For each ML algorithm, ranging in Logistic regression (LR), SVM, k-Nearest Neighbours (k-NN), linear classifier learned with stochastic gradient descent (SGD), a set of meta-features is learned to predict whether some configuration $\theta_1$ outperforms some configuration $\theta_2$ on dataset $\z$: to each triplet $(\z,\theta_1,\theta_2)$ is associated a binary value $\ell(\z,\theta_1,\theta_2)$, set to 1 iff $\theta_2$ yields better performance than $\theta_1$ on $\z$.
\XX{} parameters $\zeta$ are trained to build model $\hat \ell_\zeta$, minimizing the (weighted version of) cross-entropy loss (\ref{eq:binaryloss}), where $\hat \ell_\zeta(\z,\theta_1,\theta_2)$ is a 2-layer FC network with input vector $(\Ff_\zeta(\z);\theta_1;\theta_2)$, depending on the considered ML algorithm and its configuration space.
%
%A ranking setting is thus considered, using a cross-entropy loss on triplets $(\z,\theta_1,\theta_2)$. 
%
%Considering a ranking setting allows to overcome the issue of high varying accuracies across datasets, dwarfing the difference of accuracy among best and less good configurations for this dataset ~\citep{brendel}. 
%The natural option, considering a regression setting, is hindered as the accuracy widely varies depending on the dataset, dwarfing the difference of accuracy among best and less good configurations for this dataset (see also \citep{Brendel,Autosklearn}). 

In each epoch, a batch made of triplets $(\z,\theta_1,\theta_2)$ is built, with $\theta_1, \theta_2$ uniformly drawn in the algorithm configuration space (Table \ref{tab:hp_space}) and $\z$ a $n$-sample $d$-dimensional patch of a  dataset in the OpenML CC-2018 \citep{bischl2019openml} with $n$ uniformly drawn in $[700; 900]$ and $d$ in $[3; 10]$.

\begin{table}[t]
   \centering
    \resizebox{\columnwidth}{!}{%
   \begin{tabular}{|c|c|c|c|c|}\toprule
     Method      & SGD  &  SVM  &   LR  & k-NN \\ \midrule
     Hand-crafted & 71.18 $\pm$ 0.41 & 75.39 $\pm$ 0.29 & 86.41 $\pm$ 0.419 & 65.44 $\pm$ 0.73 \\ \midrule
     
     \textsc{DSS} (Linear aggregation) & 73.46 $\pm$ 1.44 & 82.91 $\pm$ 0.22 & 87.93 $\pm$ 0.58 & 70.07 $\pm$ 2.82 \\
     \textsc{DSS} (Equivariant+Invariant) & 73.54 $\pm$ 0.26 & 81.29 $\pm$ 1.65 & 87.65 $\pm$ 0.03 & 68.55 $\pm$ 2.84 \\
     \textsc{DSS} (Non-linear aggregation) & 74.13 $\pm$ 1.01 & 83.38 $\pm$ 0.37 & 87.92 $\pm$ 0.27 & 73.07 $\pm$ 0.77 \\ \midrule
     
     \textsc{DIDA} & \textbf{78.41} $\pm$ \textbf{0.41} & \textbf{84.14} $\pm$ \textbf{0.02} & \textbf{89.77} $\pm$ \textbf{0.50} & \textbf{81.82} $\pm$ \textbf{0.91} \\

    \bottomrule 
   \end{tabular}
   }
   \caption{Pairwise ranking of configurations, for ML algorithms SGD, SVM, LR and k-NN: performance on test set of \XX{}, hand-crafted and DSS (average and std deviation on 3 runs).}
   \label{tab:mis_ranking}
\end{table}

The quality of the \XX\ meta-features is assessed from the ranking accuracy (Table \ref{tab:mis_ranking}), showing their relevance. The performance gap compared to the baselines is higher for the k-NN modelling task; this is explained as the sought performance model only depends on the local geometry of the examples. Still, good performances are observed over all considered algorithms.
%Particularly good results are obtained on the k-NN (respectively, LR) modelling task; this is explained as the sought performance model only depends on the local geometry  (resp. the separability) of the examples. Still, good performances are observed over all four algorithms.
%Considering a ranking (instead of regression) setting allows to overcome the issue of high varying accuracies across datasets, dwarfing the difference of accuracy among best and less good configurations for this dataset ~\citep{brendel}. 
%
%Nevertheless, 
A regression setting, where a real-valued $\ell(\Ff_\zeta(\z),\theta)$ learns the predicted performance, can be successfully considered on top of the learned meta-features $\Ff_\zeta(\z)$ (illustrated on the k-NN algorithm on Figure \ref{fig:scatter_1}; other results are presented in \AppendRes{}). 
%
%Nevertheless, performance models (one for each algorithm) can be considered on top of these learnt meta-features, along a standard regression setting. %complementary experiments aimed to learn a performance model  based  . %have been considered.
%
%Figure \ref{fig:scatter_1} gives an example of the obtained results for the k-NN algorithm, while others are detailed in \AppendRes{}.
%

\begin{figure*}
\centering
\includegraphics[width=0.80\linewidth]{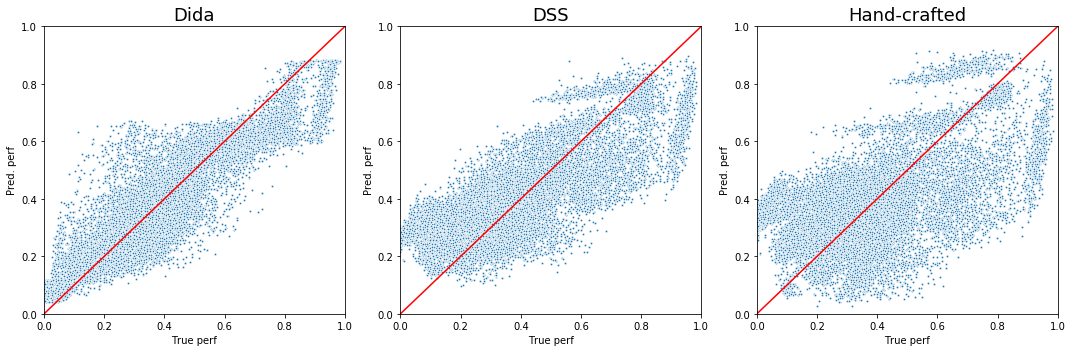}

\caption{k-NN: True performance vs performance predicted by regression on top of the meta-features (i) learned by \XX{}, (ii) DSS or (iii) Hand-crafted statistics.%, on the k-NN algorithm.
\label{fig:scatter_1}}
\end{figure*}

%% file: sections/05-Conclu.tex
%%%%%%%%%%%%%%%%%%%%%%%%%%%%%%%%%%%%%%%%%%%%%%%%%%%
\section{Conclusion}\label{sec:conclusion}

The theoretical contribution of the paper is the \XX\ architecture, able to learn from discrete and continuous distributions on $\RR^d$, invariant w.r.t. feature ordering, agnostic w.r.t. the size and dimension $d$ of the considered distribution sample (with $d$ less than some upper bound $D$). This architecture enjoys universal approximation and robustness properties, generalizing former results obtained for point clouds \citep{maron2020learning}.
The merits of \XX\ are demonstrated on two tasks defined at the dataset level: patch identification and performance model learning, comparatively to the state of the art  \citep{maron2020learning,jomaa2019dataset2vec,munoz2018instance}.
%A lesson learned is that meta-feature learning $-$ like learning in general $-$ relies on good tasks to be considered and sufficiently many examples to be available: quite a few of our early attempts failed due to current ML benchmarks being not sufficiently representative. This issue could be side-stepped by randomly sampling patches in OpenML datasets; but principled strategies to derive extensive ML benchmarking suites would definitely allow to go farther. 
%
The ability to accurately describe a dataset in the landscape defined by ML algorithms opens new perspectives to compare datasets and algorithms, e.g. for domain adaptation \citep{BenDavid07, BenDavid10} and meta-learning \citep{finn2018,yoon2018}.

%Being able to describe a dataset via meta-features guided by ML algorithm performance modelling, opens new perspectives in order to compare datasets and algorithms. Another perspective is to investigate the relationships between two datasets, and estimate {\em a priori} the chances of a successful domain adaptation \cite{BenDavid07,BenDavid10}. 

%% file: sections/acknowledgements.tex
\section*{Acknowledgements}

The work of G. De Bie is supported by the R\'egion Ile-de-France.
H. Rakotoarison acknowledges funding from the ADEME \#1782C0034 project NEXT.
The work of G. Peyr\'e was supported by the European Research Council (ERC project NORIA) and by the French government under management of Agence Nationale de la Recherche as part of the ``Investissements d’avenir'' program, reference ANR19-P3IA-0001 (PRAIRIE 3IA Institute).

%% file: sections/App-01-continuous.tex
\section{Extension to arbitrary distributions}\label{appendix:notations}

\def\MMd{\mbox{$\mathcal{P}(\RR^d)$}}
\def\MMr{\mbox{$\mathcal{P}(\RR^r)$}}
\def\MMO{\mbox{$\mathcal{P}(\Omega)$}}
\def\MMOq{\mbox{$\mathcal{P}(\Omega)_{/\sim}$}}
\paragraph{Overall notations.}
Let $X \in \Rr(\RR^d)$ denote a random vector on $\RR^d$ with  $\al_X \in \MMd$ its law (a positive Radon measure with unit mass). By definition,  its expectation denoted $\EE(X)$ reads $\EE(X) = \int_{\RR^d} x \mathrm{d}\al_X(x) \in \RR^d$, and for any continuous function $f:\RR^d\to\RR^r$, $\EE(f(X)) = \int_{\RR^d} f(x) \mathrm{d}\al_X(x)$. 
In the following, two random vectors $X$ and $X'$ with same law $\al_X$ are considered indistinguishable, noted $X{'}\sim X$.
Letting $f : \RR^d \mapsto \RR^r$ denote a function on $\RR^d$, the push-forward operator by $f$, noted $f_\sharp: \MMd \mapsto \MMr$ is defined as follows, for any $g$ continuous function from $\RR^d$ to $\RR^r$ ($g$ in $\Cc(\RR^d;\RR^r)$): 
    $$\forall g \in \Cc(\RR^d;\RR^r) \quad 
	\int_{\RR^r} g \mathrm{d}(f_\sharp \al) \eqdef \int_{\RR^d} g(f(x)) \mathrm{d}\al(x)$$	
Letting $\{x_i\}$ be a set of points in $\RR^d$ with $w_i \geq 0$ such that $\sum_i w_i=1$, the discrete measure $\al_X=\sum_i w_i \delta_{x_i}$ is the sum of the Dirac measures $\delta_{x_i}$ weighted by $w_i$.

\paragraph{Invariances.}
In this paper, we consider functions on probability measures that are \emph{invariant with respect to permutations of coordinates}. Therefore, denoting $S_d$ the $d$-sized permutation group, we consider measures over a symmetrized compact $\Omega \subset \RR^d$ equipped with the following equivalence relation: for $\alpha$, $\beta \in \MMO, \alpha \sim \beta \iff \exists \sigma \in S_d, \beta=\sigma_\sharp \alpha$, such that a measure and its permuted counterpart are indistinguishable in the corresponding quotient space, denoted alternatively \MMOq\ or $\Rr(\Omega)_{/\sim}$. A function $\phi: \Omega^n \to \RR$ is said to be invariant (by permutations of coordinates) iff $\forall \sigma \in S_d, \phi(x_1,\hdots,x_n)=\phi(\sigma(x_1),\hdots,\sigma(x_n))$ (Definition \ref{def}). 
\paragraph{Tensorization.} Letting $X$ and $Y$ respectively denote two random vectors on $\Rr(\RR^d)$ and $\Rr(\RR^p)$, the tensor product vector $X\otimes Y$ is defined as: $X\otimes Y\eqdef (X^{'},Y^{'}) \in \Rr(\RR^d\times\RR^p)$, where $X^{'}$ and $Y^{'}$ are independent and have the same law as $X$ and $Y$, i.e. $\d(\al_{X\otimes Y})(x,y)=\d\al_X(x)\d\al_Y(y)$. In the finite case, for  $\al_X=\frac{1}{n}\sum_i \delta_{x_i}$ and $\al_Y=\frac{1}{m}\sum_j \delta_{y_j}$, then $\al_{X\otimes Y}=\frac{1}{nm}\sum_{i,j} \delta_{x_i,y_j}$, weighted sum of Dirac measures on all pairs  $(x_i,y_j)$. The $k-$fold tensorization of a random vector $X\sim\al_X$, with law $\al_X^{\otimes k}$, generalizes the above construction to the case of $k$ independent random variables with law $\al_X$. Tensorization will be used to define the law of datasets, and  design universal architectures (Appendix \ref{appendix:thm}).

\paragraph{Invariant layers.} In the general case, a $G$-invariant layer $f_\phi$ with invariant map $\phi:\RR^d\times\RR^d \to \RR^r$ such that $\phi$ satisfies
$$ \forall (x_1,x_2)\in(\RR^d)^2, \forall\sigma \in G, \phi(\sigma(x_1),\sigma(x_2))=\phi(x_1,x_2)
$$ is defined as 
$$ f_\phi:X \in \Rr(\RR^d)_{/\sim}\mapsto \EE_{X'\sim X}[\phi(X,X')] \in \Rr(\RR^r)_{/\sim}
$$
where the expectation is taken over $X' \sim X$. Note that considering the couple $(X,X')$ of independent random vectors $X'\sim X$ amounts to consider the tensorized law $\al_X\otimes\al_X$.

\begin{rem}
Taking as input a discrete distribution $\al_X = \sum_{i=1}^n w_i \delta_{x_i}$, the invariant layer outputs another discrete distribution $\al_Y=\sum_{i=1}^n w_i \delta_{y_i}$ with $y_i= \sum_{j=1}^n w_j \phi(x_i,x_j)$; each input point $x_i$ is mapped onto $y_i$ summarizing the pairwise interactions with $x_i$ after $\phi$. 
\end{rem}

\begin{rem}\label{rem:group}(Generalization to arbitrary invariance groups) The definition of invariant  $\phi$ can be generalized to arbitrary invariance groups operating on $\RR^d$, in particular sub-groups of the permutation group $S_d$. 
After \citep{maron2020learning} (Thm 5), a simple and only way to design an invariant linear function is to consider $\phi(z,z')=\psi(z+z')$ with $\psi$ being $G$-invariant. 
How to design invariant functions in the general non-linear case is left for further work.
\end{rem}

\begin{rem}
Invariant layers can also be generalized to handle higher order interactions functionals, namely $f_{\phi}(X)\eqdef \EE_{X_2,\hdots,X_N\sim X}[\phi(X,X_2,\hdots,X_N)]$, which amounts to consider, in the discrete case, $N$-uple of inputs points $(x_{j_1},\hdots,x_{j_N}).$
\end{rem}

%% file: sections/App-02-proofs.tex
\section{Proofs on Regularity}\label{appendix:reg}
\def\MMd{\mbox{$\mathcal{P}(\RR^d)$}}

\paragraph{Wasserstein distance.} The regularity of the involved functionals is measured w.r.t. the $1$-Wasserstein distance between two probability distributions $(\al,\be) \in \MMd$
\begin{equation*}
	\Wass_1(\al,\be) \eqdef \umin{ \pi_1=\al,\pi_2=\be } \int_{\RR^d \times \RR^d} \norm{x-y} \d\pi(x,y) \eqdef \umin{ X\sim\al,Y\sim\be } \EE(\norm{X-Y})
\end{equation*}
where the minimum is taken over measures on $\RR^d \times \RR^d$ with marginals $\al,\be \in \MMd$. $\Wass_{1}$ is known to be a norm \citep{santambrogio2015optimal}, that can be conveniently computed using
\eq{
	\Wass_1(\al,\be) = \Wass_1(\al-\be) = \umax{\Lip(g) \leq 1} \int_{\RR^d} g \d(\al-\be),
}
where $\Lip(g)$ is the Lipschitz constant of $g : \RR^d \rightarrow \RR$ with respect to the Euclidean norm (unless otherwise stated). For simplicity and by abuse of notations, $\Wass_1(X,Y)$ is used instead of $\Wass_1(\al,\be)$ when $X\sim\al$ and $Y\sim\be$. %This  abuse of notation makes sense since $X\sim X^{'}$ are indistinguishable. 
The convergence in law denoted $\rightharpoonup$ is equivalent to the convergence in Wasserstein distance in the sense that $X_k \rightharpoonup X$ is equivalent to $\Wass_1(X_k,X) \rightarrow 0$.

\paragraph{Permutation-invariant Wasserstein distance.} The Wasserstein distance is quotiented according to the permutation-invariance equivalence classes: for $\al,\be \in \MMd$ 
$$\overline{\Wass}_1(\al,\be) \eqdef \min_{\sigma \in S_d} \Wass_1(\sigma_\sharp \al,\be)=\min_{\sigma \in S_d} \umax{\Lip(g) \leq 1} \int_{\RR^d} g \circ \sigma \d\al-\int_{\RR^d} g \d\be
$$
such that $\overline{\Wass}_1(\al,\be)=0 \iff \al\sim\be$. $\overline{\Wass}_1$ defines a norm on $\MMd_{/\sim}$.
%, hence the abuse of notations in the following. It also metrizes convergence in law, bearing in mind that measures are known up to their equivalence classes.

\paragraph{Lipschitz property.} A map $f : \Rr(\RR^d)
%_{/\sim} 
\rightarrow \Rr(\RR^r)
%_{/\sim}
$ is continuous for the convergence in law (aka the weak$^*$ of measures) if for any sequence $X_k \rightharpoonup X$, then $f(X_k) \rightharpoonup f(X)$. 
Such a map is furthermore said to be $C$-Lipschitz for the permutation invariant 1-Wasserstein distance if 
\begin{align}
	\foralls (X,Y) \in (\Rr(\RR^d)_{/\sim})^2, \, 
	\overline{\Wass}_1( f(X),f(Y) )
	\leq 
	C \overline{\Wass}_1(X,Y).
\end{align}
Lipschitz properties enable us to analyze robustness to input perturbations, since it ensures that if the input distributions of random vectors are close in the permutation invariant Wasserstein sense, the corresponding output laws are close, too. %\todo{Gwen: remis ici pour les classes d'équivalences (eq 8 vs 3), si reviewer pointilleux.}

%%%%%%%%%% Je m'arrete ici

\paragraph{Proofs of section \ref{subsec:reg}.}
\begin{proof}(Proposition \ref{prop:lip}). For $\al, \be \in \MMd{}$, Proposition 1 from \citep{de2019stochastic} yields $\Wass_1(f_\phi(\al),f_\phi(\be)) \leq 2r\Lip(\phi)\Wass_1(\al,\be)$, hence, for $\sigma \in \G,$
\begin{align*}
    \Wass_1(\sigma_\sharp f_\phi(\al),f_\phi(\be)) & \leq \Wass_1(\sigma_\sharp f_\phi(\al),f_\phi(\al)) + \Wass_1(f_\phi(\al),f_\phi(\be))  \\
    & \leq \Wass_1(\sigma_\sharp f_\phi(\al),f_\phi(\al)) + 2r\Lip(\phi)\Wass_1(\al,\be)
\end{align*}
hence, taking the infimum over $\sigma$ yields
\begin{align*}
    \overline{\Wass}_1(f_\phi(\al),f_\phi(\be)) & \leq \overline{\Wass}_1(f_\phi(\al),f_\phi(\al)) + 2r\Lip(\phi)\Wass_1(\al,\be) \\
    & \leq 2r\Lip(\phi)\Wass_1(\al,\be)
\end{align*}
Since $f_\phi$ is invariant, for $\sigma \in \G$, $f_\phi(\z)=f_\phi(\sigma_\sharp \z)$,
\begin{align*}
    \overline{\Wass}_1(f_\phi(\al),f_\phi(\be)) & \leq 2r\Lip(\phi)\Wass_1(\sigma_\sharp\al,\be)
\end{align*}
Taking the infimum over $\sigma$ yields the result.
\end{proof}

\begin{proof}(Proposition \ref{prop:lipAE}). 
To upper bound
$\overline{\Wass}_1(\xi_\sharp f_\phi(\tau_\sharp\al),f_\phi(\al))$ for $\al\in\MMd{}$, we proceed as follows, using proposition 3 from \citep{de2019stochastic} and proposition \ref{prop:lip}:
\begin{align*}
    \Wass_1(\xi_\sharp f_\phi(\tau_\sharp\al_\phi(\al)),f_\phi(\al)) &\leq \Wass_1(\xi_\sharp f_\phi(\tau_\sharp\al),f_\phi(\tau_\sharp \al)) + \Wass_1(f_\phi(\tau_\sharp \al),f_\phi(\al)) \\
    & \leq \norm{\xi-id}_{L^1(f_\phi(\tau_\sharp \al))} + \Lip(f_\phi) \Wass_1(\tau_\sharp \al, \al) \\
    & \leq \sup_{y \in f_\phi(\tau(\Omega))} \norm{\xi(y)-y}_2 + 2r \Lip(\phi)\sup_{x\in\Omega} \norm{\tau(x)-x}_2
\end{align*}
For $\sigma \in \G$, we get
\begin{align*}
    \Wass_1(\sigma_\sharp \xi_\sharp f_\phi(\tau_\sharp\al),f_\phi(\al)) & \leq \Wass_1(\sigma_\sharp \xi_\sharp f_\phi(\tau_\sharp\al),\xi_\sharp f_\phi(\tau_\sharp\al)) + \Wass_1(\xi_\sharp f_\phi(\tau_\sharp\al),f_\phi(\al))
\end{align*}
Taking the infimum over $\sigma$ yields
\begin{align*}
    \overline{\Wass}_1(\xi_\sharp f_\phi(\tau_\sharp\al),f_\phi(\al)) & \leq \Wass_1(\xi_\sharp f_\phi(\tau_\sharp\al),f_\phi(\al)) \\
    &\leq \sup_{y \in f_\phi(\tau(\Omega))} \norm{\xi(y)-y}_2 + 2rC(\phi)\sup_{x\in\Omega} \norm{\tau(x)-x}_2
\end{align*}

Similarly, for $\al,\be\in(\MMd{})^2,$
\begin{align*}
    \Wass_1(\xi_\sharp f_\phi(\tau_\sharp\al),\xi_\sharp f_\phi(\tau_\sharp\be)) 
    & \leq \Lip(\xi) \Wass_1(f_\phi(\tau_\sharp\al),f_\phi(\tau_\sharp\be)) \\
    & \leq \Lip(\xi) \Lip(f_\phi) \Wass_1(\tau_\sharp\al,\tau_\sharp\be) \\
    & \leq 2r \Lip(\phi) \Lip(\xi) \Lip(\tau) \Wass_1(\al,\be)\\
\end{align*}
hence, for $\sigma \in \G$,
\begin{align*}
    \Wass_1(\sigma_\sharp \xi_\sharp f_\phi(\tau_\sharp\al), \xi_\sharp f_\phi(\tau_\sharp\be)) & \leq \Wass_1(\sigma_\sharp \xi_\sharp f_\phi(\tau_\sharp\al), \xi_\sharp f_\phi(\tau_\sharp\al)) \\
    & \qquad + \Wass_1(\xi_\sharp f_\phi(\tau_\sharp\al), \xi_\sharp f_\phi(\tau_\sharp\be))
\end{align*}
and taking the infimum over $\sigma$ yields
\begin{align*}
    \overline{\Wass}_1(\xi_\sharp f_\phi(\tau_\sharp\al), \xi_\sharp f_\phi(\tau_\sharp\be)) & \leq \Wass_1(\xi_\sharp f_\phi(\tau_\sharp\al), \xi_\sharp f_\phi(\tau_\sharp\be)) \\
    &\leq 2r \Lip(\phi) \Lip(\xi) \Lip(\tau) \Wass_1(\al,\be)
\end{align*}
Since $\tau$ is equivariant: namely, for $\al \in \MMd{},$ $\sigma \in \G, \tau_\sharp (\sigma_\sharp \al)=\sigma_\sharp (\tau_\sharp \al),$ hence, since $f_\phi$ is invariant, $f_\phi(\tau_\sharp (\sigma_\sharp \al)) = f_\phi(\sigma_\sharp (\tau_\sharp \al)) = f_\phi(\tau_\sharp \al),$ hence for $\sigma \in \G,$
\begin{align*}
    \overline{\Wass}_1(\xi_\sharp f_\phi(\tau_\sharp\al), \xi_\sharp f_\phi(\tau_\sharp\be)) &\leq 2r \Lip(\phi) \Lip(\xi) \Lip(\tau) \Wass_1(\sigma_\sharp \al,\be)
\end{align*}
Taking the infimum over $\sigma$ yields the result.
\end{proof}

\section{Proofs on Universality}\label{appendix:thm}
\paragraph{Detailed proof of Theorem \ref{thm:approx-main}.}
This paragraph details the result in the case of $S_d-$invariance, while the next one focuses on invariances w.r.t. products of permutations. Before providing a proof of Theorem \ref{thm:approx-main} we first state two useful lemmas. Lemma \ref{lem-old} is mentioned for completeness, referring the reader to \cite{de2019stochastic}, Lemma 1 for a proof.

\begin{lem}\label{lem-old}
	Let $\left( S_j \right)_{j=1}^N$ be a partition of a domain including $\Omega$ ($S_j \subset \RR^d$) and let $x_j \in S_j$.
	Let $( \phi_j )_{j=1}^N$ a set of bounded functions $\phi_j : \Om \rightarrow \RR$ supported on $S_j$, such that $\sum_j \phi_j=1$ on $\Om$. 
	For $\al \in \MMO{}$, we denote $\hat{\al}_N \eqdef \sum_{j=1}^N \al_j \delta_{x_j}$ with $\al_j \eqdef \int_{S_j} \phi_j\d\al$. One has, denoting $\De_j \eqdef \max_{x \in S_j} \norm{x_j-x}$, 
	\eq{
		\Wass_1( \hat{\al}_N, \al) \leq \max_{1 \leq j \leq N} \De_j.
	}
\end{lem}

\begin{lem}\label{lem-holder}
	Let $f:\mathbb{R}^d\to\mathbb{R}^q$ a $1/p$-Hölder continuous function $(p\geq 1)$, then there exists a constant $C>0$ such that for all $\alpha, \beta \in \MMd{}$, $\Wass_1(f_\sharp\alpha,f_\sharp\beta) \leq C \Wass_1(\alpha,\beta)^{1/p}$.
\end{lem}
\begin{proof}
For any transport map $\pi$ with marginals $\al$ and $\beta$, $1/p$-Hölderness of $f$ with constant $C$ yields $\int \lvert\lvert f(x)-f(y) \rvert\rvert_2 \mathrm{d}\pi(x,y) \leq C \int \lvert\lvert x-y \rvert\rvert_2^{1/p} \mathrm{d}\pi(x,y) \leq C \left( \int \lvert\lvert x-y \rvert\rvert_2 \mathrm{d}\pi(x,y) \right)^{1/p}$ using Jensen's inequality ($p\leq 1$). Taking the infimum over $\pi$ yields $\Wass_1(f_\sharp\alpha,f_\sharp\beta) \leq C \Wass_1(\alpha,\beta)^{1/p}$.
\end{proof}

Now we are ready to dive into the proof. Let $\al \in \MMd{}.$ We consider:
\begin{itemize}
    \item[$\bullet$] $h: x=(x_1,\hdots,x_d) \in \mathbb{R}^{d}\mapsto \left( \sum_{1 \leq j_1 < \hdots < j_i \leq d} x_{j_1} \cdot \hdots \cdot x_{j_i}  \right)_{i=1\hdots d} \in \mathbb{R}^{d}$ the collection of $d$ elementary symmetric polynomials; $h$ does not lead to a loss in information, in the sense that it generates the ring of $S_d$-invariant polynomials (see for instance \cite{cox2007ideals}, chapter 7, theorem 3) while preserving the classes (see the proof of Lemma 2, appendix D from \cite{maron2020learning});
    
    \item[$\bullet$] $h$ is obviously not injective, so we consider $\pi:\mathbb{R}^d \to \mathbb{R}^d/S_d$ the projection onto $\mathbb{R}^d/S_d$: $h=\Tilde{h}\circ \pi$ such that $\Tilde{h}$ is bijective from $\pi(\Omega)$ to its image $\Omega^{'}$, compact of $\mathbb{R}^d$; $\Tilde{h}$ and ${\Tilde{h}^{-1}}$ are continuous;
    
    \item[$\bullet$] Let $(\phi_i)_{i=1\hdots N}$ the piecewise affine P1 finite element basis, which are hat functions on a discretization $(S_i)_{i=1\hdots N}$ of $\Omega^{'} \subset \mathbb{R}^{d}$, with centers of cells $(y_i)_{i=1\hdots N}$. We then define
    $g: x \in \mathbb{R}^{d} \mapsto (\phi_1(x),\hdots,\phi_N(x)) \in \mathbb{R}^{N}$;
    
    \item[$\bullet$] $f:(\alpha_1,\hdots,\alpha_N) \in \mathbb{R}^{N} \mapsto \Ff\left(\sum_{i=1}^N\alpha_i \delta_{\Tilde{h}^{-1}(y_i)}\right) \in \mathbb{R}$.
\end{itemize}

We approximate $\Ff$ using the following steps: 
\begin{itemize}
    \item[$\bullet$] Lemma \ref{lem-old} (see Lemma 1 from \cite{de2019stochastic}) yields that ${h}_\sharp\al$ and $\widehat{h_\sharp\al}=\sum_{i=1}^N \alpha_i\delta_{y_i}$ are close: $\Wass_1({h}_\sharp\al,\widehat{h_\sharp\al}) \leq \sqrt{d}/N^{1/d}$;
    
    \item[$\bullet$] The map ${\Tilde{h}^{-1}}$ is regular enough ($1/d$-Hölder) such that according to Lemma \ref{lem-holder}, there exists a constant $C > 0$ such that
    $$\Wass_1({\Tilde{h}^{-1}}_\sharp({h}_\sharp\al),{\Tilde{h}^{-1}}_\sharp\widehat{h_\sharp\al}) \leq C \Wass_1({h}_\sharp\al,\widehat{h_\sharp\al})^{1/d} \leq C d^{1/2d}/N^{1/d^2}$$ 
    Hence $\overline{\Wass}_1(\al,{\Tilde{h}^{-1}}_\sharp\widehat{h_\sharp\al}) := \inf_{\sigma \in S_d} \Wass_1(\sigma_\sharp \al,{\Tilde{h}^{-1}}_\sharp\widehat{h_\sharp\al}) \leq C d^{1/2d}/N^{1/d^2}.$
    
    Note that $h$ maps the roots of polynomial $\prod_{i=1}^d (X-x^{(i)})$ to its coefficients (up to signs). Theorem 1.3.1 from \cite{Rahman2002AnalyticTO} yields continuity and $1/d$-Hölderness of the reverse map. Hence ${\Tilde{h}^{-1}}$ is $1/d$-Hölder.
    
    \item[$\bullet$] Since $\Omega$ is compact, by Banach-Alaoglu theorem, we obtain that $\MMO{}$ is weakly-* compact, hence  $\MMOq{}$ also is. Since $\Ff$ is continuous, it is thus uniformly weak-* continuous: for any $\epsilon >0$, there exists $\delta > 0$ such that $\overline{\Wass}_1(\al,{\Tilde{h}^{-1}}_\sharp\widehat{h_\sharp\al}) \leq \delta$ implies $\lvert \Ff(\alpha)-\Ff({\Tilde{h}^{-1}}_\sharp\widehat{h_\sharp\al}) \rvert < \epsilon$. Choosing $N$ large enough such that $C d^{1/2d}/N^{1/d^2} \leq \delta$ therefore ensures that $\lvert \Ff(\alpha)-\Ff({\Tilde{h}^{-1}}_\sharp\widehat{h_\sharp\al}) \rvert < \epsilon$.
\end{itemize}

\paragraph{Extension of Theorem \ref{thm:approx-main} to products of permutation groups.}

\begin{cor} \label{cor:product}
    Let $\Ff:\MMOq{} \to \RR$ a continuous $S_{d_1}\times\hdots\times S_{d_n}$-invariant map ($\sum_i d_i=d$), where $\Omega$ is a symmetrized compact over $\RR^d$. Then $\forall \epsilon > 0$, there exists three continuous maps $f, g, h$ such that
    $$ \forall \al \in \Mm_+^1(\Omega)_{/ \sim}, \lvert \Ff(\al)-f \circ \EE \circ g (h_\sharp\al)  \rvert < \epsilon 
    $$ where $h$ is invariant; $g, h$ are independent of $\Ff$.
\end{cor}

\begin{proof} We provide a proof in the case $G=S_d\times S_p$, which naturally extends to any product group $G=S_{d_1}\times\hdots\times S_{d_n}$. We trade $h$ for the collection of elementary symmetric polynomials in the first $d$ variables; and in the last $p$ variables: $h:(x_1,\hdots,x_d,y_1,\hdots,y_p) \in \RR^{d+p}\mapsto( [ \sum_{1 \leq j_1 < \hdots < j_i \leq d} x_{j_1} \hdots x_{j_i}  ]_{i=1}^d; [ \sum_{1 \leq j_1 < \hdots < j_i \leq p} y_{j_1}  \hdots y_{j_i}  ]_{i=1}^p  ) \in \RR^{d+p}$ up to normalizing constants (see Lemma \ref{lem:concatholder}). Step 1 (in Lemma \ref{lem:productring}) consists in showing that $h$ does not lead to a loss of information, in the sense that it generates the ring of $S_d\times S_p-$invariant polynomials. In step 2 (in Lemma \ref{lem:concatholder}), we show that $\Tilde{h}^{-1}$ is $1/\max(d,p)-$Hölder. Combined with the proof of Theorem \ref{thm:approx-main}, this amounts to showing that the concatenation of Hölder functions (up to normalizing constants) is Hölder. With these ingredients, the sketch of the previous proof yields the result.\end{proof}

\begin{lem} \label{lem:productring}
Let the collection of symmetric invariant polynomials $[P_i(X_1,\hdots,X_d)]_{i=1}^d \eqdef [ \sum_{1 \leq j_1 < \hdots < j_i \leq d} X_{j_1} \hdots X_{j_i}  ]_{i=1}^d$ and $[Q_i(Y_1,\hdots,Y_p)]_{i=1}^p=[ \sum_{1 \leq j_1 < \hdots < j_i \leq p} Y_{j_1}  \hdots Y_{j_i}  ]_{i=1}^p$. The $d+p-$sized family $(P_1,\hdots,P_d,Q_1,\hdots,Q_p)$ generates the ring of $S_d\times S_p-$invariant polynomials.
\end{lem}
\begin{proof}
The result comes from the fact the fundamental theorem of symmetric polynomials (see \cite{cox2007ideals} chapter 7, theorem 3) does not depend on the base field. Every $S_d\times S_p-$invariant polynomial $P(X_1,\hdots,X_d,Y_1,\hdots,Y_p)$ is also $S_d\times I_p-$invariant with coefficients in $\RR[Y_1,\hdots,Y_p]$, hence it can be written $P=R(Y_1,\hdots,Y_p)(P_1,\hdots,P_d)$. It is then also $S_p-$invariant with coefficients in $\RR[P_1,\hdots,P_d]$, hence it can be written $P=S(Q_1,\hdots,Q_p)(P_1,\hdots,P_d) \in \RR[P_1,\hdots,P_d,Q_1,\hdots,Q_p]$.
\end{proof}

\begin{lem} \label{lem:concatholder}
Let $h:(x,y)\in\Omega\subset\RR^{d+p}\mapsto(f(x)/C_1,g(y)/C_2)\in\RR^{d+p}$ where $\Omega$ is compact, $f:\RR^d\to\RR^d$ is $1/d-$Hölder with constant $C_1$ and $g:\RR^p\to\RR^p$ is $1/p-$Hölder with constant $C_2$. Then $h$ is $1/\max(d,p)-$Hölder.
\end{lem}
\begin{proof}
Without loss of generality, we consider $d>p$ so that $\max(d,p)=d$, and $f,g$ normalized (f.i. $\forall x,x_0 \in (\RR^d)^2, \norm{f(x)-f(x_0)}_1\leq \norm{x-x_0}_1^{1/d}$). For $(x,y),(x_0,y_0) \in \Omega^2$, $\norm{h(x,y)-h(x_0,y_0)}_1\leq\norm{f(x)-f(x_0)}_1+\norm{g(y)-g(y_0)}_1\leq \norm{x-x_0}_1^{1/d}+\norm{y-y_0}_1^{1/p}$ since both $f,g$ are Hölder. We denote $D$ the diameter of $\Omega$, such that both $\norm{x-x_0}_1/D\leq 1$ and $\norm{y-y_0}_1/D\leq 1$ hold. Therefore $\norm{h(x,y)-h(x_0,y_0)}_1 \leq D^{1/d} \left(\frac{\norm{x-x_0}_1}{D}\right)^{1/d} + D^{1/p} \left(\frac{\norm{y-y_0}_1}{D}\right)^{1/p} \leq 2^{1-1/d}D^{1/p-1/d}\norm{(x,y)-(x_0,y_0)}_1^{1/d}$ using Jensen's inequality, hence the result.
\end{proof}

In the next two paragraphs, we focus the case of $S_d-$invariant functions for the sake of clarity, without loss of generality. Indeed, the same technique applies to $\G-$invariant functions as $h$ in that case has the same structure: its first $d_X$ components are $S_{d_X}-$invariant functions of the first $d_X$ variables and its last $d_Y$ components are $S_{d_Y}-$invariant functions of the last variables.

\paragraph{Extension of Theorem \ref{thm:approx-main} to distributions on spaces of varying dimension.}
\begin{cor} \label{cor:uniform}
    Let $I=[0;1]$ and, for $k \in [1;d_m], \Ff_k:\MMIk{}\to\RR$ continuous and $S_k-$invariant. Suppose $(\Ff_k)_{k=1\hdots d_m-1}$ are restrictions of $\Ff_{d_m}$, namely, $\forall \al_k \in \MMIk{}, \Ff_k(\al_k)=\Ff_{d_m}(\al_k\otimes\delta_0^{\otimes d_m-k})$. Then functions $f$ and $g$ from Theorem \ref{thm:approx-main} are uniform: there exists $f,g$ continuous, $h_1,\hdots,h_{d_m}$ continuous invariant such that $$\forall k=1\hdots d_m, \forall \al_k \in \MMIk{}, \lvert \Ff_k(\al_k)-f\circ \EE \circ g ({h_k}_{\sharp}\al_k) \rvert < \epsilon.$$
\end{cor}

\begin{proof}
Theorem \ref{thm:approx-main} yields continuous $f,g$ and a continuous invariant $h_{d_m}$ such that $\forall \al \in \MMIdm{}, \lvert \Ff_{d_m} - f\circ \EE\circ g ({h_{d_m}}_\sharp \al) \rvert < \epsilon.$ For $k=1\hdots d_m-1$, we denote $h_k:(x_1,\hdots,x_k)\in\RR^k\mapsto  (( \sum_{1 \leq j_1 < \hdots < j_i \leq k} x^{(j_1)} \cdot \hdots \cdot x^{(j_i)})_{i=1\hdots k},0\hdots,0)\in\RR^{d_m}$. With the hypothesis, for $k=1\hdots d_m-1$, $\al_k \in \MMIk{},$ the fact that ${h_k}_\sharp (\al_k)={h_{d_m}}_\sharp (\al_k\otimes \delta_0^{\otimes d_m-k})$ yields the result.
\end{proof}

\paragraph{Approximation by invariant neural networks.}
Based on theorem \ref{thm:approx-main}, $\Ff$ is uniformly close to $f \circ \mathbb{E} \circ g \circ h$:
\begin{itemize}
\item We approximate $f$ by a neural network $f_\theta:x\in \mathbb{R}^N\mapsto C_1 \lambda(A_1 x + b_1) \in \mathbb{R}$, where $p_1$ is an integer, $A_1 \in \mathbb{R}^{p_1\times N}, C_1 \in \mathbb{R}^{1\times p_1}$ are weights, $b_1 \in \mathbb{R}^{p_1}$ is a bias and $\lambda$ is a non-linearity. 
\item Since each component $\phi_j$ of $\phi=g\circ h$ is permutation-invariant, it has the representation $\phi_j:x=(x_1,\hdots,x_d)\in\RR^d\mapsto \rho_j \left( \sum_{i=1}^d u(x_{i}) \right)$ \cite{zaheer2017} (which is a special case of our layers with a base function only depending on its first argument, see section \ref{inv-sto-lay}), $\rho_j:\mathbb{R}^{d+1}\to\mathbb{R}$, and $u:\mathbb{R}\to\mathbb{R}^{d+1}$ independent of $j$ (see \cite{zaheer2017}, theorem 7). 
\item We can approximate $\rho_j$ and $u$ by neural networks $\rho_{j,\theta}:x\in \mathbb{R}^{d+1}\mapsto C_{2,j} \lambda(A_{2,j} x + b_{2,j}) \in \mathbb{R}$ and $u_{\theta}:x\in \mathbb{R}^d\mapsto C_3 \lambda(A_3 x + b_3) \in \mathbb{R}^{d+1}$, where $p_{2,j}, p_3$ are integers, $A_{2,j} \in \mathbb{R}^{p_{2,j}\times (d+1)}, C_{2,j} \in \mathbb{R}^{1\times p_{2,j}}, A_3 \in \mathbb{R}^{p_3 \times 1}, C_3 \in \mathbb{R}^{(d+1)\times p_3}$ are weights and $b_{2,j} \in \mathbb{R}^{p_{2,j}}, b_3 \in \mathbb{R}^{p_3}$ are biases, and denote $\phi_\theta(x)=(\phi_{j,\theta}(x))_j\eqdef (\rho_{j,\theta}(\sum_{i=1}^d u_\theta(x_i)))_j$. 
\end{itemize}
Indeed, we  upper-bound the difference of interest $\lvert \Ff(\al) - f_\theta \left( \mathbb{E}_{X\sim\al} \left( \phi_\theta(X) \right) \right) \rvert$ by triangular inequality by the sum of three terms:
\begin{itemize}
    \item $\lvert \Ff(\al) - f \left( \mathbb{E}_{X\sim\al} \left( \phi(X) \right) \right) \rvert$
    \item $\lvert f \left( \mathbb{E}_{X\sim\al} \left( \phi(X) \right) \right) - f_\theta \left( \mathbb{E}_{X\sim\al} \left( \phi(X) \right) \right) \rvert$
    \item $\lvert f_\theta \left( \mathbb{E}_{X\sim\al} \left( \phi(X) \right) \right) - f_\theta \left( \mathbb{E}_{X\sim\al} \left( \phi_\theta(X) \right) \right) \rvert$
\end{itemize}
and bound each term by $\frac{\epsilon}{3}$, which yields the result.
The bound on the first term directly comes from theorem \ref{thm:approx-main} and yields a constant $N$ which depends on $\epsilon$. The bound on the second term is a direct application of the universal approximation theorem (UAT) \cite{cybenko1989approximation,leshno1993multilayer}. Indeed, since $\al$ is a probability measure, input values of $f$ lie in a compact subset of $\RR^N$: $\norm{ \int_\Omega g\circ h(x) \mathrm{d}\alpha }_\infty \leq \max_{x\in\Omega} \max_i | g_i \circ h(x) |$, hence the theorem is applicable as long as $\la$ is a nonconstant, bounded and continuous activation function. 
Let us focus on the third term.
Uniform continuity of $f_\theta$ yields the existence of $\delta > 0$ s.t. $\left\lvert \left\lvert u-v \right\rvert \right\rvert_1 < \delta$ implies $\left\lvert f_\theta(u)-f_\theta(v) \right\rvert < \frac{\epsilon}{3}$. Let us apply the UAT: each component $\phi_j$ of $h$ can be approximated by a neural network $\phi_{j,\theta}$. Therefore:
\begin{align*}
&\norm{ \EE_{X \sim \al}  \left( \phi(X) - \phi_\theta(X) \right) }_1 
 \leq \EE_{X \sim \al} \norm{ \phi(X) - \phi_\theta(X) }_1 \leq \sum_{j=1}^N \int_\Omega | \phi_j(x) - \phi_{j,\theta}(x) | \mathrm{d}\al(x)  \\
 & \qquad \leq \sum_{j=1}^N \int_\Omega \lvert \phi_j(x) - \rho_{j,\theta}(\sum_{i=1}^d u(x_i)) \rvert\mathrm{d}\al(x) \\
& \qquad+ \sum_{j=1}^N \int_\Omega \lvert \rho_{j,\theta}(\sum_{i=1}^d u(x_i)) - \rho_{j,\theta}(\sum_{i=1}^d u_\theta(x_i)) \rvert \mathrm{d}\al(x) \\
 & \qquad \leq N\frac{\delta}{2N} + N\frac{\delta}{2N}=\delta
\end{align*}
using the triangular inequality and the fact that $\al$ is a probability measure. The first term is small by UAT on $\rho_j$ while the second also is, by UAT on $u$ and uniform continuity of $\rho_{j,\theta}$. Therefore, by uniform continuity of $f_\theta$, we can conclude.

\paragraph{Universality of tensorization.}

This complementary theorem provides insight into the benefits of tensorization for approximating invariant regression functionals, as long as the test function is invariant.
\begin{thm}\label{thm:sw}
    The algebra \eq{
\Aa_\Omega \eqdef 
\left\{
\Ff:\MMOq{} \to \RR, \exists n \in \NN, \exists \phi : \Omega^n \to \RR \text{ invariant}, \forall \al, \Ff(\al) = \int_{\Omega^n} \phi \mathrm{d}\al^{\otimes n} 
\right\}
}
where $\otimes n$ denotes the $n$-fold tensor product, is dense in $\Cc(\Mm_+^1(\Omega)_{/ \sim}).$
\end{thm}

\begin{proof}
This result follows from the Stone-Weierstrass theorem. Since $\Omega$ is compact, by Banach-Alaoglu theorem, we obtain that $\MMO{}$ is weakly-* compact, hence  $\MMOq{}$ also is. In order to apply Stone-Weierstrass, we show that $\Aa_\Omega$ contains a non-zero constant function and is an algebra that separates points. A (non-zero, constant) 1-valued function is obtained with $n=1$ and $\phi=1$. Stability by scalar is straightforward. For stability by sum: given $(\Ff_1,\Ff_2) \in \Aa_\Omega^2$ (with associated  functions $(\phi_1,\phi_2)$ of tensorization degrees $(n_2,n_2)$), we denote $n \eqdef \max(n_1,n_2)$ and $\phi(x_1,\hdots,x_{n}) \eqdef \phi_1(x_1,\hdots,x_{n_1}) + \phi_2(x_1,\hdots,x_{n_2})$ which is indeed invariant, hence $\Ff_1+\Ff_2 =  \int_{\Omega^n} \phi \mathrm{d}\al^{\otimes n} \in \Aa_\Omega$. Similarly, for stability by product: denoting this time $n = n_1+n_2$, we introduce the invariant $\phi(x_1,\hdots,x_{n}) = \phi_1(x_1,\hdots,x_{n_1})\times\phi_2(x_{n_1+1},\ldots,x_{n})$, which shows that $\Ff=\Ff_1 \times \Ff_2 \in \Aa_\Omega$ using Fubini's theorem. Finally, $\mathcal{A}_\Omega$ separates points: if $\al \neq \nu$, then there exists a symmetrized domain $S$ such that $\al(S)\neq\nu(S)$: indeed, if for all symmetrized domains $S$, $\al(S)=\nu(S)$, then $\al(\Omega)=\nu(\Omega)$ which is absurd. Taking $n=1$ and $\phi=1_{S}$ (invariant since $S$ is symmetrized) yields an $\Ff$ such that $\Ff(\al)\neq\Ff(\nu)$.
\end{proof}

%% file: sections/App-03-expe.tex
\section{Experimental validation, supplementary material}\label{appendix:expe}

\XX\ and DSS source code are provided in the last file of the supplementary material.

\subsection{Benchmark Details}
Three benchmarks are used (Table \ref{tab:dataset}): TOY and UCI, taken from \citep{jomaa2019dataset2vec}, and OpenML CC-18. TOY includes 10,000 datasets, where instances are distributed along mixtures of Gaussian, intertwinning moons and rings in $\RR^2$, with 2 to 7 classes. UCI includes 121 datasets from the UCI Irvine repository \citep{UCI}. Each benchmark is divided into 70\%-30\% training-test sets.

\begin{table}[!ht]
\begin{center}
\scalebox{0.9}{
\begin{tabular}{lccccc}
\hline
            & \multicolumn{1}{l}{\# datasets} & \multicolumn{1}{l}{\# samples} & \multicolumn{1}{l}{\# features} & \multicolumn{1}{l}{\# labels} & test ratio \\ \hline
Toy Dataset & 10000                           & {[}2048, 8192{]}               & 2                               & {[}2, 7{]}   &   0.3             \\ \hline
UCI         & 121                             & {[}10, 130064{]}               & {[}3, 262{]}                    & {[}2, 100{]}     & 0.3            \\ \hline
OpenML CC-18   &  71                               &      [500, 100000]                          &                        [5, 3073]        &       [2, 46] &  0.5   \\ \hline
\end{tabular}}
\caption{Benchmarks characteristics}\label{tab:dataset}
\end{center}
\end{table}

\subsection{Baseline Details}

\paragraph{Dataset2Vec details.} We used the available implementation of \Baseline{}\footnote{\Baseline\ code is available at \url{https://github.com/hadijomaa/dataset2vec}}. In the experiments, \Baseline{} hyperparameters are set to their default values except size and number of patches, set to same values as in \XX.

\paragraph{\textsc{DSS} layer details.} We built our own implementation of invariant \textsc{DSS} layers, as follows.
Linear invariant DSS layers (see \citep{maron2020learning}, Theorem 5, 3.) are of the form \begin{align} \label{eq:maroninv}
    L_{inv}:X \in \RR^{n\times d} \mapsto L^H(\sum_{j=1}^n x_j) \in \RR^K
\end{align}
where $L^H:\RR^d\to\RR^K$ is a linear $H$-invariant function.
Our applicative setting requires that the implementation accommodates to varying input dimensions $d$ as well as permutation invariance, hence we consider the Deep Sets representation (see \citep{zaheer2017}, Theorem 7)
\begin{align} \label{eq:maronH}
    L^H:x=(x_1,\hdots,x_d)\in\RR^d \mapsto \rho \left( \sum_{i=1}^d \phi(x_i) \right) \in \RR^K
\end{align}
where $\phi:\RR\to\RR^{d+1}$ and $\rho:\RR^{d+1}\to\RR^K$ are modelled as (i) purely linear functions; (ii) FC networks, which extends the initial linear setting (\ref{eq:maroninv}).
In our case, $H=S_{d_X}\times S_{d_Y}$, hence, two invariant layers of the form (\ref{eq:maroninv}-\ref{eq:maronH}) are combined to suit both feature- and label-invariance requirements.
Both outputs are concatenated and followed by an FC network to form the DSS meta-features.
The last experiments use \textsc{DSS} equivariant layers (see \citep{maron2020learning}, Theorem 1), which take the form 
\begin{align} \label{eq:maroneq}
    L_{eq}:X \in \RR^{n\times d} \mapsto \left( L^1_{eq}(x_i) + L^2_{eq}(\sum_{j\neq i} x_j) \right)_{i \in [n]} \in \RR^{n\times d}
\end{align}
where $L^1_{eq}$ and $L^2_{eq}$ are linear $H$-equivariant layers.
Similarly, both feature- and label-equivariance requirements are handled via the Deep Sets representation of equivariant functions (see \citep{zaheer2017}, Lemma 3) and concatenated to be followed by an invariant layer, forming the DSS meta-features.
All methods are allocated the same number of parameters to ensure fair comparison.
We provide our implementation of the DSS layers in the supplementary material.

\paragraph{Hand-crafted meta-features.}
For the sake of reproducibility, the list of \mf s used in section \ref{sec:expe} is given in Table  \ref{tab:mf_handcrafted}. Note that
meta-features related to missing values and categorical features are omitted, as being irrelevant for the considered benchmarks. Hand-crafted \mf s are extracted using the BYU \texttt{metalearn} library\footnote{See \url{https://github.com/byu-dml/metalearn}}.

\begin{table}[!ht]
\centering
\scalebox{0.9}{
\begin{tabular}{|l|l|l|l|}
\hline
\textbf{Meta-features} & \textbf{Mean} & \textbf{Min} & \textbf{Max} \\ \hline
Quartile2ClassProbability & 0.500 & 0.75 & 0.25 \\ \hline
MinorityClassSize & 487.423 & 426.000 & 500.000 \\ \hline
Quartile3CardinalityOfNumericFeatures & 224.354 & 0.000 & 976.000 \\ \hline
RatioOfCategoricalFeatures & 0.347 & 0.000 & 1.000 \\ \hline
MeanCardinalityOfCategoricalFeatures & 0.907 & 0.000 & 2.000 \\ \hline
SkewCardinalityOfNumericFeatures & 0.148 & -2.475 & 3.684 \\ \hline
RatioOfMissingValues & 0.001 & 0.000 & 0.250 \\ \hline
MaxCardinalityOfNumericFeatures & 282.461 & 0.000 & 977.000 \\ \hline
Quartile2CardinalityOfNumericFeatures & 185.555 & 0.000 & 976.000 \\ \hline
KurtosisClassProbability & -2.025 & -3.000 & -2.000 \\ \hline
NumberOfNumericFeatures & 3.330 & 0.000 & 30.000 \\ \hline
NumberOfInstancesWithMissingValues & 2.800 & 0.000 & 1000.000 \\ \hline
MaxCardinalityOfCategoricalFeatures & 0.917 & 0.000 & 2.000 \\ \hline
Quartile1CardinalityOfCategoricalFeatures & 0.907 & 0.000 & 2.000 \\ \hline
MajorityClassSize & 512.577 & 500.000 & 574.000 \\ \hline
MinCardinalityOfCategoricalFeatures & 0.879 & 0.000 & 2.000 \\ \hline
Quartile2CardinalityOfCategoricalFeatures & 0.915 & 0.000 & 2.000 \\ \hline
NumberOfCategoricalFeatures & 1.854 & 0.000 & 27.000 \\ \hline
NumberOfFeatures & 5.184 & 4.000 & 30.000 \\ \hline
Dimensionality & 0.005 & 0.004 & 0.030 \\ \hline
SkewCardinalityOfCategoricalFeatures & -0.050 & -4.800 & 0.707 \\ \hline
KurtosisCardinalityOfCategoricalFeatures & -1.244 & -3.000 & 21.040 \\ \hline
StdevCardinalityOfNumericFeatures & 68.127 & 0.000 & 678.823 \\ \hline
StdevClassProbability & 0.018 & 0.000 & 0.105 \\ \hline
KurtosisCardinalityOfNumericFeatures & -1.060 & -3.000 & 12.988 \\ \hline
NumberOfInstances & 1000.000 & 1000.000 & 1000.000 \\ \hline
Quartile3CardinalityOfCategoricalFeatures & 0.916 & 0.000 & 2.000 \\ \hline
NumberOfMissingValues & 2.800 & 0.000 & 1000.000 \\ \hline
Quartile1ClassProbability & 0.494 & 0.463 & 0.500 \\ \hline
StdevCardinalityOfCategoricalFeatures & 0.018 & 0.000 & 0.707 \\ \hline
MeanClassProbability & 0.500 & 0.500 & 0.500 \\ \hline
NumberOfFeaturesWithMissingValues & 0.003 & 0.000 & 1.000 \\ \hline
MaxClassProbability & 0.513 & 0.500 & 0.574 \\ \hline
NumberOfClasses & 2.000 & 2.000 & 2.000 \\ \hline
MeanCardinalityOfNumericFeatures & 197.845 & 0.000 & 976.000 \\ \hline
SkewClassProbability & 0.000 & -0.000 & 0.000 \\ \hline
Quartile3ClassProbability & 0.506 & 0.500 & 0.537 \\ \hline
MinCardinalityOfNumericFeatures & 138.520 & 0.000 & 976.000 \\ \hline
MinClassProbability & 0.487 & 0.426 & 0.500 \\ \hline
RatioOfInstancesWithMissingValues & 0.003 & 0.000 & 1.000 \\ \hline
Quartile1CardinalityOfNumericFeatures & 160.748 & 0.000 & 976.000 \\ \hline
RatioOfNumericFeatures & 0.653 & 0.000 & 1.000 \\ \hline
RatioOfFeaturesWithMissingValues & 0.001 & 0.000 & 0.250 \\ \hline
\end{tabular}}
\caption{Hand-crafted \mf s}
\label{tab:mf_handcrafted}
\end{table}

\subsection{Performance Prediction}

\paragraph{Experimental setting.} Table \ref{tab:hp_space} details all hyper-parameter configurations $\Theta$ considered in Section \ref{subsec:perflearning}.
As said, the learnt meta-features $\Ff_\zeta(\z)$ can be used in a regression setting, predicting the performance of various ML algorithms on a dataset $\z$.
Several performance models have been considered on top of the meta-features learnt in Section \ref{subsec:perflearning}, for instance (i) a BOHAMIANN network \citep{bohamiann2016}; (ii) Random Forest models, trained under a Mean Squared Error loss between predicted and true performances. 

\paragraph{Results.}
Table \ref{tab:MSE_BOH} reports the Mean Squared Error on the test set with performance model BOHAMIANN \citep{bohamiann2016}, comparatively to DSS and hand-crafted ones. Replacing the surrogate model with Random Forest concludes to the same ranking as in Table \ref{tab:MSE_BOH}.
Figure \ref{fig:scatter} complements Table \ref{tab:MSE_BOH} in assessing the learnt \XX{} meta-features for performance model learning. It shows \XX{}'s ability to capture more expressive meta-features than both DSS and hand-crafted ones, for all ML algorithms considered.

\begin{table}
\begin{tabular}{|l|l|c|c|}
\hline
                                                   & \multicolumn{1}{c|}{\textbf{Parameter}} & \textbf{Parameter values}                              & \multicolumn{1}{c|}{\textbf{Scale}} \\ \hline
\multicolumn{1}{|c|}{\multirow{6}{*}{\textbf{LR}}} & warm start                             & True, Fase                                              & \multicolumn{1}{c|}{}               \\ \cline{2-4} 
\multicolumn{1}{|c|}{}                             & fit intercept                          & True, Fase                                              & \multicolumn{1}{c|}{}               \\ \cline{2-4} 
\multicolumn{1}{|c|}{}                             & tol                                    & {[}0.00001, 0.0001{]}                                   & \multicolumn{1}{c|}{}               \\ \cline{2-4} 
\multicolumn{1}{|c|}{}                             & C                                      & {[}1e-4, 1e4{]}                                         & \multicolumn{1}{c|}{log}            \\ \cline{2-4} 
\multicolumn{1}{|c|}{}                             & solver                                 & newton-cg, lbfgs, liblinear, sag, saga                  & \multicolumn{1}{c|}{}               \\ \cline{2-4} 
\multicolumn{1}{|c|}{}                             & max\_iter                              & {[}5, 1000{]}                                           & \multicolumn{1}{c|}{}               \\ \hline
\multirow{7}{*}{\textbf{SVM}}                      & kernel                                 & linear, rbf, poly, sigmoid                              & \multicolumn{1}{c|}{}               \\ \cline{2-4} 
                                                   & C                                      & {[}0.0001, 10000{]}                                     & log                                 \\ \cline{2-4} 
                                                   & shrinking                              & True, False                                             &                                     \\ \cline{2-4} 
                                                   & degree                                 & {[}1, 5{]}                                              &                                     \\ \cline{2-4} 
                                                   & coef0                                  & {[}0, 10{]}                                             &                                     \\ \cline{2-4} 
                                                   & gamma                                  & {[}0.0001, 8{]}                                         &                                     \\ \cline{2-4} 
                                                   & max\_iter                              & {[}5, 1000{]}                                           &                                     \\ \hline
\multirow{3}{*}{\textbf{KNN}}                      & n\_neighbors                           & {[}1, 100{]}                                            & log                                 \\ \cline{2-4} 
                                                   & p                                      & {[}1, 2{]}                                              &                                     \\ \cline{2-4} 
                                                   & weights                                & uniform, distance                                       &                                     \\ \hline
\multirow{11}{*}{\textbf{SGD}}                     & alpha                                  & {[}0.1,  0.0001{]}                                      & log                                 \\ \cline{2-4} 
                                                   & average                                & True, False                                             &                                     \\ \cline{2-4} 
                                                   & fit\_intercept                         & True, False                                             &                                     \\ \cline{2-4} 
                                                   & learning rate                          & optimal, invscaling, constant                           &                                     \\ \cline{2-4} 
                                                   & loss                                   & hinge, log, modified\_huber, squared\_hinge, perceptron &                                     \\ \cline{2-4} 
                                                   & penalty                                & l1, l2, elasticnet                                      &                                     \\ \cline{2-4} 
                                                   & tol                                    & {[}1e-05, 0.1{]}                                        & log                                 \\ \cline{2-4} 
                                                   & eta0                                   & {[}1e-7, 0.1{]}                                         & log                                 \\ \cline{2-4} 
                                                   & power\_t                               & {[}1e-05, 0.1{]}                                        & log                                 \\ \cline{2-4} 
                                                   & epsilon                                & {[}1e-05, 0.1{]}                                        & log                                 \\ \cline{2-4} 
                                                   & l1\_ratio                              & {[}1e-05, 0.1{]}                                        & log                                 \\ \hline
\end{tabular}
\caption{Hyper-parameter configurations considered in Section \ref{subsec:perflearning}.} \label{tab:hp_space}
\end{table}

\begin{comment} % details BO
$S(\Ff_\zeta(_z), \theta) \in \RR$ 
S model BOHAMIANN
data train: ${(\Ff_\zeta(z_{train}), \theta, perf)}$
data test: ${(\Ff_\zeta(z_{test}), \theta, perf)}$
https://ml.informatik.uni-freiburg.de/papers/16-NIPS-BOHamiANN.pdf
pred_{(z_test, \theta)} = BOHAMIANN($\Ff_\zeta(z_{test}), \theta$)
\end{comment}

\begin{table}
   \centering
    \resizebox{\columnwidth}{!}{%
   \begin{tabular}{|c|c|c|c|c|}\toprule
     Method      & SGD  &  SVM  &   LR  & KNN \\ \midrule
     Hand-crafted & 0.016 $\pm$ 0.001 & 0.021 $\pm$ 0.001 & 0.018 $\pm$ 0.002 & 0.034 $\pm$ 0.001 \\ \midrule
     
     \textsc{DSS} (Linear aggregation) & 0.015 $\pm$ 0.007 & 0.020 $\pm$ 0.002 & 0.019 $\pm$ 0.001 & 0.025 $\pm$ 0.010 \\
     \textsc{DSS} (Equivariant+Invariant) & 0.014 $\pm$ 0.002 & 0.017 $\pm$ 0.003 & 0.015 $\pm$ 0.003 & 0.028 $\pm$ 0.003 \\
     \textsc{DSS} (Non-linear aggregation) & 0.015 $\pm$ 0.009 & 0.016 $\pm$ 0.003 & 0.014 $\pm$ 0.001 & 0.020 $\pm$ 0.005 \\ \midrule
     
     \textsc{DIDA} & \textbf{0.012} $\pm$ \textbf{0.001} & \textbf{0.015} $\pm$ \textbf{0.001} & \textbf{0.010} $\pm$ \textbf{0.001} & \textbf{0.009} $\pm$ \textbf{0.000} \\

    \bottomrule 
   \end{tabular}
   }
   \caption{Performance modelling, comparative results of \XX{}, DSS and Hand-crafted (HC) \mf s: Mean Squared Error (average over 5 runs) on test set, between the true performance and the performance predicted by the trained BOHAMIANN surrogate model, for  ML algorithms SVM, LR, kNN, SGD (see text).}
   \label{tab:MSE_BOH}
\end{table}

\begin{figure}
\centering
\subfigure[]{\includegraphics[width=0.9\textwidth]{knn_dida_dss_hc.png}}
\subfigure[]{\includegraphics[width=0.9\textwidth]{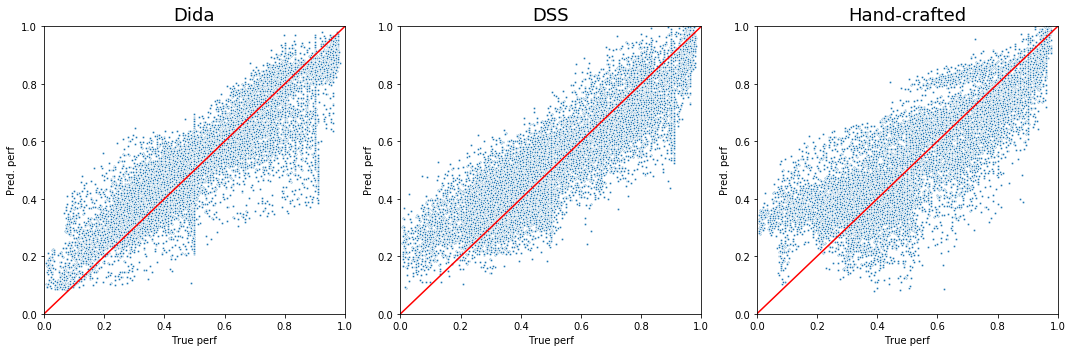}}
\subfigure[]{\includegraphics[width=0.9\textwidth]{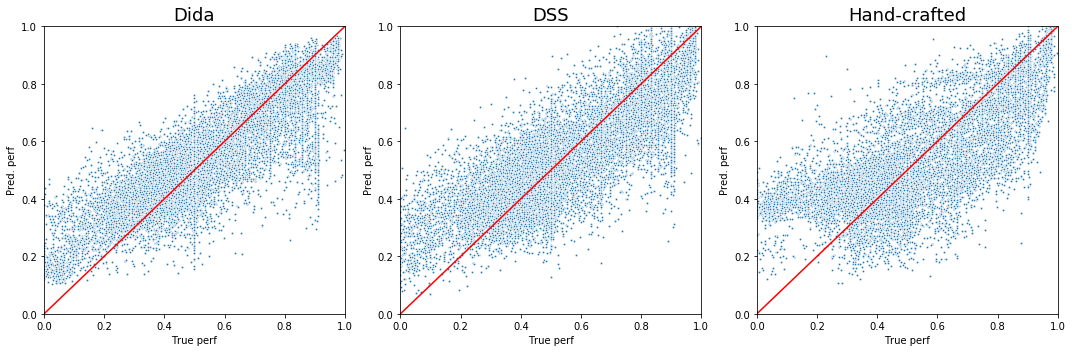}}
\subfigure[]{\includegraphics[width=0.9\textwidth]{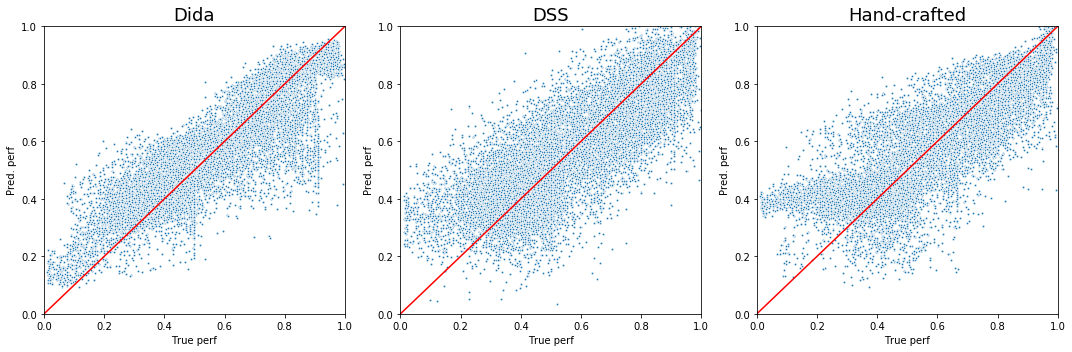}}
\caption{Comparison between the true performance and the performance predicted by the trained surrogate model on \XX{}, DSS or Hand-crafted \mf s, for various ML algorithms: (a) k-NN; (b) Logistic Regression; (c) SVM; (d) Linear classifier learnt with stochastic gradient descent.\label{fig:scatter}}
\end{figure}